\documentclass{article}

\usepackage[preprint,nonatbib]{neurips_2020}

\usepackage[utf8]{inputenc}
\usepackage[T1]{fontenc}
\usepackage[hidelinks]{hyperref}
\usepackage{url}
\usepackage{booktabs}
\usepackage{amsfonts}
\usepackage{nicefrac}
\usepackage{microtype}

\usepackage{algorithm}
\usepackage{algorithmic}
\usepackage{caption}

\usepackage{amsmath}
\usepackage{amssymb}
\usepackage{amsthm}
\usepackage{bbm}
\usepackage{bm}
\usepackage{boxedminipage}
\usepackage{float}
\usepackage{ifthen}
\usepackage{listings}
\usepackage{mathtools}
\usepackage{multirow}
\usepackage{todonotes}
\usepackage{xparse}
\usepackage{xspace}

\usepackage{textcase}
\usepackage{soul}
\usepackage{indentfirst}

\usepackage{tikz}

\newtheorem*{rep@theorem}{\rep@title}
\newcommand{\newreptheorem}[2]{%
\newenvironment{rep#1}[1]{%
 \def\rep@title{\theoremref{##1} Restated}%
 \begin{rep@theorem}}%
 {\end{rep@theorem}}}
\makeatother
\makeatletter
\newtheorem*{rep@lemma}{\rep@title}
\newcommand{\newreplemma}[2]{%
\newenvironment{rep#1}[1]{%
 \def\rep@title{\lemmaref{##1} Restated}%
 \begin{rep@lemma}}%
 {\end{rep@lemma}}}
\makeatother
\makeatletter
\newtheorem*{rep@claim}{\rep@title}
\newcommand{\newrepclaim}[2]{%
\newenvironment{rep#1}[1]{%
 \def\rep@title{\claimref{##1} Restated}%
 \begin{rep@claim}}%
 {\end{rep@claim}}}
\makeatother
\makeatletter
\newtheorem*{rep@problem}{\rep@title}
\newcommand{\newrepproblem}[2]{%
\newenvironment{rep#1}[1]{%
 \def\rep@title{\problemref{##1} Restated}%
 \begin{rep@problem}}%
 {\end{rep@problem}}}
\makeatother
\newtheorem{theorem}{Theorem}
\newreptheorem{theorem}{Theorem}
\newtheorem{importedtheorem}{Imported Theorem}
\newtheorem{problem}{Problem}
\newrepproblem{problem}{Problem}

\newtheorem{definition}{Definition}
\newtheorem{lemma}{Lemma}
\newreplemma{lemma}{Lemma}
\newtheorem{claim}{Claim}
\newrepclaim{claim}{Claim}
\newtheorem{corollary}[theorem]{Corollary}

\newcommand\numberthis{\addtocounter{equation}{1}\tag{\theequation}}

\newcommand{\namedref}[2]{\texorpdfstring{\hyperref[#2]{#1~\ref*{#2}}}{#1~\ref*{#2}}\xspace}
\newcommand{\lemmaref}[1]{\namedref{Lemma}{lem:#1}}
\newcommand{\theoremref}[1]{\namedref{Theorem}{thm:#1}}
\newcommand{\claimref}[1]{\namedref{Claim}{clm:#1}}
\newcommand{\corolref}[1]{\namedref{Corollary}{corol:#1}}

\newcommand{\equationref}[1]{\namedref{Equation}{eq:#1}}
\newcommand{\inequalityref}[1]{\namedref{Inequality}{ineq:#1}}
\newcommand{\defref}[1]{\namedref{Definition}{def:#1}}

\newcommand{\problemref}[1]{\namedref{Problem}{prob:#1}}
\newcommand{\algorithmref}[1]{\namedref{Algorithm}{alg:#1}}
\newcommand{\importedtheoremref}[1]{\namedref{Imported Theorem}{impthm:#1}}

\newcommand{\sectionref}[1]{\namedref{Section}{sec:#1}}
\newcommand{\appendixref}[1]{\namedref{Appendix}{app:#1}}

\newcommand{\naive}{\text{na\"ive}\xspace}

\newcommand{\floor}[1]{\ensuremath{\left\lfloor{#1}\right\rfloor}\xspace}
\newcommand{\abs}[1]{\ensuremath{\vert{#1}\vert}\xspace}

\newcommand{\eps}[0]{\ensuremath{\varepsilon}}
\let\epsilon\eps

\newcommand{\cA}{\ensuremath{{\mathcal A}}\xspace}
\newcommand{\cB}{\ensuremath{{\mathcal B}}\xspace}
\newcommand{\cC}{\ensuremath{{\mathcal C}}\xspace}

\newcommand{\cF}{\ensuremath{{\mathcal F}}\xspace}
\newcommand{\cG}{\ensuremath{{\mathcal G}}\xspace}
\newcommand{\cH}{\ensuremath{{\mathcal H}}\xspace}
\newcommand{\cI}{\ensuremath{{\mathcal I}}\xspace}
\newcommand{\cK}{\ensuremath{{\mathcal K}}\xspace}
\newcommand{\cL}{\ensuremath{{\mathcal L}}\xspace}

\newcommand{\cS}{\ensuremath{{\mathcal S}}\xspace}

\newcommand{\cU}{\ensuremath{{\mathcal U}}\xspace}

\newcommand{\by}{\ensuremath{{\mathbf y}}\xspace}

\newcommand{\bA}{\ensuremath{{\mathbf A}}\xspace}

\newcommand{\bbC}{\ensuremath{{\mathbb C}}\xspace}

\newcommand{\bbE}{\ensuremath{{\mathbb E}}\xspace}

\newcommand{\bbR}{\ensuremath{{\mathbb R}}\xspace}

\newcommand{\defeq}[0]{\ensuremath{\;{\vcentcolon=}\;}\xspace}

\newcommand{\E}[0]{\mathop{\bbE}\xspace}

\DeclareMathOperator{\trace}{tr}

\DeclareMathOperator{\sinc}{sinc}

\newcommand{\mat}[1]{\boldsymbol{#1}}
\renewcommand{\vec}[1]{\boldsymbol{\mathrm{#1}}}
\newcommand{\vecalt}[1]{\boldsymbol{#1}}

\newcommand{\normof}[1]{\|#1\|}

\newcommand{\bmat}[1]{\begin{bmatrix} #1 \end{bmatrix}}

\newcommand{\sbmat}[1]{\left[\begin{smallmatrix} #1 \end{smallmatrix}\right]}

\newcommand{\RR}{\mathbb{R}}

\newcommand{\eye}{\mat{I}\xspace}
\newcommand{\mA}{\ensuremath{\mat{A}}\xspace}
\newcommand{\mB}{\ensuremath{\mat{B}}\xspace}

\newcommand{\mF}{\ensuremath{\mat{F}}\xspace}

\newcommand{\mI}{\ensuremath{\mat{I}}\xspace}

\newcommand{\mK}{\ensuremath{\mat{K}}\xspace}

\newcommand{\mS}{\ensuremath{\mat{S}}\xspace}

\newcommand{\vb}{\ensuremath{\vec{b}}\xspace}
\newcommand{\vc}{\ensuremath{\vec{c}}\xspace}

\newcommand{\vk}{\ensuremath{\vec{k}}\xspace}

\newcommand{\vv}{\ensuremath{\vec{v}}\xspace}
\newcommand{\vw}{\ensuremath{\vec{w}}\xspace}
\newcommand{\vx}{\ensuremath{\vec{x}}\xspace}
\newcommand{\vy}{\ensuremath{\vec{y}}\xspace}

\newcommand{\vsigma}{\ensuremath{\vecalt{\sigma}}\xspace}
\newcommand{\valpha}{\ensuremath{\vecalt{\alpha}}\xspace}

\sodef\allcapsspacing{\upshape}{0.15em}{0.65em}{0.6em}%

\colorlet{todo_background_normal}{white}
\definecolor{todo_background_dark}{RGB}{39,40,34}

\definecolor{advice_text}{RGB}{78, 12, 123}
\colorlet{advice_background}{todo_background_normal}

\definecolor{incomplete_text}{RGB}{204, 64, 84}
\colorlet{incomplete_background}{todo_background_normal}

\newcounter{question}
\DeclareMathOperator*{\argmin}{argmin}

\usepackage{enumitem}

\title{The Statistical Cost of Robust Kernel Hyperparameter Tuning}

\author{%
  Raphael A.~Meyer \\
  Tandon School of Engineering\\
  New York University\\
  \texttt{ram900@nyu.edu} \\
  \And
  Christopher Musco \\
  Tandon School of Engineering\\
  New York University\\
  \texttt{cmusco@nyu.edu} \\
}

\begin{document}

\maketitle

\begin{abstract}
This paper studies the statistical complexity of kernel hyperparameter tuning in the setting of active regression under adversarial noise.
We consider the problem of finding the best interpolant from a class of kernels with unknown hyperparameters, assuming only that the noise is square-integrable.
We provide finite-sample guarantees for the problem, characterizing how increasing the complexity of the kernel class increases the complexity of learning kernel hyperparameters.
For common kernel classes (e.g. squared-exponential kernels with unknown lengthscale), our results show that hyperparameter optimization increases sample complexity by just a logarithmic factor, in comparison to the setting where optimal parameters are known in advance.
Our result is based on a subsampling guarantee for linear regression under multiple design matrices, combined with an $\epsilon$-net argument for discretizing kernel parameterizations.  
\end{abstract}

\section{Introduction}

In machine learning, Kernel Ridge Regression (KRR) is central to modern time series analysis and nonparametric regression.
For time series, Gaussian Processes model the covariance of a stochastic process using a kernel matrix, and interpolate the underlying signal with KRR \cite{DBLP:books/lib/RasmussenW06}.
In nonparametric regression, kernels define a local-averaging scheme, and KRR provides a smooth interpolation for the function \cite{DBLP:journals/technometrics/Fotopoulos07,DBLP:books/daglib/0035708}.
Experimentally, it is known that the Kernel Ridge Regression estimator generalizes and interpolates well over continuous domains \cite{wilson2013gaussian,avron2019universal}.

However, it is also known that \emph{kernel regression only performs well when kernel hyperparameters are chosen well} \cite{DBLP:conf/icml/WilsonN15,benton2019function}.
This observation has lead to significant interest in algorithms that try to find the best kernel parameters in a large search space \cite{wang2019million,lanckriet2004learning}.
Additionally, the existing research generally assumes that observation noise is independent, unbiased, and random \cite{DBLP:books/lib/RasmussenW06,DBLP:books/daglib/0034861,DBLP:conf/icml/AvronKMMVZ17,DBLP:journals/technometrics/Fotopoulos07}.
The goal of this paper is to understand the \emph{statistical cost} of this sort of hyperparameter optimization when we can have \emph{worst-case observation noise}.
How many data samples are needed to avoid over-fitting when searching over such a large class of models?

We formalize this problem in an adversarial noise setting that originated in literature on function approximation \cite{ChenKanePrice:2016,ChenPrice:2019a,CohenMigliorati:2017}.
By Bochner's theorem, every stationary (shift invariant) kernel function $k_\mu$ can be written \(k_{\mu}(\Delta) = \int_\bbR e^{- 2 \pi i \xi \Delta} \mu(\xi) d\xi\) for some probability density function \(\mu\).
\cite{avron2019universal} introduces the following active regression problem for interpolating with a fixed \(k_\mu\):

\begin{problem}
\label{prob:fixed-kernel}
Let \(y(t)\) be a signal we wish to interpolate.
 Let \(z(t)\) be an adversarial noise signal.
Fix regularization parameter \(\eps > 0\) and
observe \(y(t) + z(t)\) at any chosen times \(t_1,\ldots,t_n\).
How large does $n$ need to be so that an interpolant \(\tilde y\) constructed from our observations satisfies:
\[
	\normof{\tilde y - y}_T^2 \leq O(1) \cdot \left(\normof{z}_T^2 + \eps \cdot \textup{Energy}_\mu(y)\right)
\]
\end{problem}
Here \(\normof{x}_T^2 \defeq \int_0^T \abs{x(t)}^2 \frac1T dt\) is the natural $\ell_2$ norm on $[0,T]$. 
\(\text{Energy}_\mu(y)\) is a natural measure of the cost of representing the ground truth signal \(y\) with the kernel \(k_\mu\), formally defined in \sectionref{technical-overview}.
Intuitively, if the kernel \(k_\mu\) cannot represent \(y\) easily, then the associated term \(\text{Energy}_{\mu}(y)\) is large, and hence the interpolation error may be large.

\problemref{fixed-kernel} is a robust, active, nonparametric learning problem.
It is nonparametric in the sense that a kernel is being used to interpolate the signal \(y\).
It is robust in the sense that the noise function \(z(t)\) is arbitrary (for instance, we do \emph{not} assume that \(z(t)\) is a zero-mean stochastic process).
It is active in the sense that the user chooses the time points \(t_1,\ldots,t_n\).

\cite{avron2019universal} shows that if we let the number of observations \(n\) exceed a natural \textit{Statistical Dimension} parameter which is a function of the kernel \(k_\mu\) and regularization parameter \(\eps\) (see \sectionref{statistical-dimension} for a formal definition), then KRR solves \problemref{fixed-kernel}.
Moreover, for many common kernels (square exponential, sinc, Lorentzian, etc.) this number of samples is necessary in the worst-case.

Since the observation noise $z$ is adversarial, a linear dependence on $\normof{z}_T^2$ is inevitable.
On the other hand, the $\textup{Energy}_\mu(y)$ term can be reduced by decreasing \eps, but this increases the statistical dimension of the problem, necessitating more samples.
Alternatively, we can decrease the energy term substantially by simply choosing a different kernel.
This is the problem of kernel hyperparameter tuning:
\begin{problem}
\label{prob:interpolate}
Let \(y(t)\) be a signal we wish to interpolate.
Let \(z(t)\) be an adversarial noise signal.
Let \cU be a (possibly infinite) set of kernel PDFs.
Fix regularization parameter \(\eps > 0\) and observe \(y(t)+z(t)\) at any chosen times \(t_1,\ldots,t_n\).
How large does \(n\) need to be so that we can select a PDF \(\tilde\mu\in\cU\) (correspondingly, a shift-invariant kernel function $k_{\tilde{\mu}}$) and construct a KRR interpolant \(\tilde y\) from our observations such that:
\[
	\normof{\tilde y - y}_T^2 \leq O(1) \cdot \left(\normof{z}_T^2 + \eps \cdot \min_{\mu\in\cU} ~ \textup{Energy}_{\mu}(y)\right)
\]
\end{problem}
We should think of \cU as containing all PDFs corresponding to kernels in a structured class: for examples, all squared exponential kernels with unknown lengthscale.
To solve \problemref{interpolate}, we must find hyperparameters that are competitive with the best possible kernel in \cU.
This is still a robust, active, nonparametric learning problem, but is now generalized to consider hyperparameters.
At a high level, our main result is to prove that the number of time samples required to solve \problemref{interpolate} is not much larger than the number of samples required to solve \problemref{fixed-kernel}.

\subsection{Prior work}

There is substantial prior work on hyperparameter tuning between the Learning Theory, Time Series, and Signal Processing literatures.
In the Learning Theory community, the problem of ``learning kernels'' is typical, but usually assumes we are given a finite set of fixed kernels and have to learn how to combine the given kernels \cite{cortes2009new,zien2007multiclass,DBLP:conf/icml/MeyerH19}.
There does exist some work that discusses tuning hyperparameters for kernel families, but these works all make iid noise assumptions \cite{DBLP:journals/neco/YingC10}.
There is also work on gradient methods for hyperparameter tuning, but this work generally avoids finite sample complexity bounds \cite{DBLP:books/lib/RasmussenW06,benton2019function}.
In signal processing, kernel hyperparameter tuning generalizes the well studied problem of spectrum-blind signal reconstruction  \cite{FengBresler:1996,Bresler:2008,MishaliEldar:2009}. However, prior work in that area again does not provide finite sample complexity bounds.
The core technical results of this paper extend tools from recent work on Randomized Signal Processing.
In particular, papers in this area deal with adversarial observation noise, but either assume we know the kernel function exactly \cite{avron2019universal} or primarily address Fourier sparse function fitting \cite{ChenKanePrice:2016}.

\subsection{Contributions}
Our main contribution is to extend work on Randomized Signal Processing to bound the sample complexity of Problem \ref{prob:interpolate}.
For many cases (i.e. the Squared Exponential Kernel with unknown lengthscale), we prove that the sample complexity of learning both the hyperparameters and the signal \(y(t)\) is only logarithmically larger than the sample complexity of learning \(y(t)\) with known hyperparameters (see \corolref{sm-bounds} with \(q=1\)). In other words, solving \problemref{interpolate} is not much harder than solving \problemref{fixed-kernel} with the hardest single kernel in \cU.
We prove this in two core steps:

\begin{itemize}[leftmargin=*]
	\item
	First, we consider the setting where we want to optimize over a large but finite set of \(Q\) possible kernels.
	To solve \problemref{fixed-kernel}, the problem where we have a fixed kernel, prior work requires the number of samples to depend \emph{linearly} on \(1/\delta\) \cite{avron2019universal}.
	Accordingly, a \naive solution to \problemref{interpolate} that combines existing results with a union bound would require the number of samples to grow linearly with \(Q\). In \sectionref{multiple-priors} we improve this dependence to be logarithmic.
	Our result requires a subsampling guarantee for linear operators that may have infinite dimension.
	When applied to finite matrices, this result corresponds to a guarantee for subsampled linear regression with multiple design matrices.
	\item Next, we show how to use this result to bound the sample complexity of hyperparameter tuning for kernels with an \textit{infinite space of hyperparameters}.
	In particular, \sectionref{discrete-hyperparameters} shows how to discretize the space of hyperparameters, reducing the problem from picking a hyperparameter in a continuous space to picking a hyperparameter from a finite set.
	Then, the result from the first bullet point bounds the actual sample complexity of learning our hyperparameters.
	For demonstration purposes, a full analysis is presented for the commonly used Spectral Mixture (SM) Kernel, but the broad framework generalizes to most other stationary kernels.
\end{itemize}
For a summary sample complexity bound for SM kernels, the reader can skip ahead to Corollary \ref{corol:sm-bounds} in Section \ref{sec:discrete-hyperparameters}. We prove that learning hyperparameters for a $q$-component SM kernel can be done with $\tilde{O}(q^2MT)$ samples from $[0,T]$, if the mixture components each have lengthscale $\leq M$. The linear dependence on $MT$ is near optimal even for $q=1$, as shown in \cite{avron2019universal}. We suspect the dependence on $q^2$ can be improved to linear, although note that $q$ is typically a small constant (e.g., $< 10$) in practice \cite{he2015state}.

\section{Preliminaries}
Let bold capital letters, like \mA and \mB, denote complex-valued matrices.
Let bold lower case letters, like \vx and \vy, denote complex-valued vectors.
\(\normof{\vx}_2\) denotes the \(\ell_2\) norm of \vx. We view infinite-dimensional linear operators as generalization of matrices, and functions as generalizations of vectors, so the notation used with be analogous.
Calligraphic capital letters, like \cA and \cB will represent either linear operators or sets; it will be clear from context.
Lower case non-bold letters, like \(f\) and \(g\), denote complex-valued functions of real numbers.
Typically \(y(t)\) and \(z(t)\) will represent functions in the time domain, while \(g(\xi)\) and \(h(\xi)\) will represent functions in the frequency domain.
We use \(\preceq\) and \(\succeq\) to denote semidefinite order for both matrices and Hermitian operators. 

In general, we use $\cH$ to denote a Hilbert space. $\langle f,g\rangle_{\cH}$ and $\|\cdot\|_\cH$ denote the corresponding inner product and norm. 
For a complex number $x$, we let \(x^*\) denote its complex conjugate. For a matrix or linear operator $\cA$, we let $\cA^*$ denote the Hermitian adjoint.
That is, if \cA maps between Hilbert spaces \(\cH_1\) and \(\cH_2\), then \(\cA^*:\cH_2\rightarrow\cH_1\) satisfies \(\langle f,\cA^*g\rangle_{\cH_1} = \langle\cA f,g\rangle_{\cH_2}\) for any \(f\in \cH_1\), \(g\in \cH_2\). 

\subsection{Shift Invariant Kernels}
This paper is concerned with shift-invariant, positive semidefinite kernel functions on the real line.
By Bochner's theorem, any such kernel is the Fourier transform of a positive measure \cite{DBLP:conf/nips/RahimiR07}, and for all settings we consider, the measure will be a probability measure with finitely bounded probability density function $\mu$.\footnote{
	Throughout this paper, $\mu$ will sometimes denote a scaled PDF that integrates to a constant other than 1.
}
We denote the corresponding kernel function by $k_\mu$: 
\begin{align}
\label{eq:bochners}
k_\mu(t_1-t_2) = \int_{\xi\in \RR} e^{-2\pi i (t_1 - t_2)} \mu(\xi) d\xi .
\end{align}

For example, when \(\mu(\xi) = \frac{1}{\sqrt{2\pi \sigma^2}}e^{-\xi^2/2\sigma^2}\) is the Gaussian density, \(k_\mu(\Delta) = e^{-\Delta^2\sigma^2}\) is a squared exponential kernel, also called a radial basis function (RBF) kernel. When \(\mu(\xi) = 1/2F\) for \(\xi \in [-F,F]\) and 0 elsewhere (a uniform density), \(k_\mu(\Delta) = \sinc(F|\Delta|)\) is a sinc kernel.

We let \(L_2(\mu)\) denote the space of complex-valued square integrable functions with respect to \(\mu\).
$L_2(\mu)$ has inner product \(\langle g,h\rangle_\mu \defeq \int_\bbR g(\xi)^* h(\xi) \mu(\xi) d\xi\) and norm
\(\normof{g}_\mu^2 \defeq \langle g,g\rangle_\mu\). 
We will also refer to \(\normof{g}_\mu^2\) as the power of \(g\) with respect to \(\mu\).
A function $g$ is in \(L_2(\mu)\) if $\normof{g}_\mu < \infty$. 
We let \(L_2(T)\) denote the set of complex-valued square integrable functions on \([0,T]\).
I.e. \(L_2(T)\) has inner product \(\langle x,y\rangle_T\defeq\int_0^T x(t)^*y(t) \frac1T dt\) and norm
\(\normof{x}_T^2 \defeq \langle x,x\rangle_T\).
A function \(x\) is in \(L_2(T)\) if \(\normof{x}_T < \infty\).

\subsection{Statistical Dimension and Universal Sampling}
\label{sec:statistical-dimension}
As discussed, the sample complexity of interpolating a function $y$ on $[0,T]$ with a \emph{fixed} kernel function $k_\mu$ is characterized by the statistical dimension of that kernel. Before formally defining this quantity, we introduce the integral operator $\cK_\mu: L_2(T)\rightarrow L_2(T)$
\begin{align*}
[\cK_\mu x](t)\defeq \int_0^T k_\mu(s-t)x(s) \frac{1}{T}ds, 
\end{align*}
which is defined for any kernel function $k_\mu$ and time range $[0,T]$. Note that $\cK_\mu = \cF_\mu^* \cF_\mu$ where $\cF_\mu$ and $\cF_\mu^*$ are the following Fourier transform and inverse Fourier transform operators:
\begin{align*}
&\cF_\mu: L_2(T) \rightarrow L_2(\mu)
&
&[\cF_\mu x](\xi) \defeq \int_0^T x(t) e^{-2 \pi i \xi t} \frac1T dt
\\
&\cF^*_\mu: L_2(\mu) \rightarrow L_2(T)
&
&[\cF_\mu^* g](t) \defeq \int_\bbR g(\xi) e^{2 \pi i \xi t} \mu(\xi) d\xi
\end{align*}

\begin{definition}[Statistical Dimension]
\label{def:stat_dim_kernel}
For any bounded PDF $\mu$ with corresponding kernel $k_\mu$, time range $[0,T]$, and parameter $\eps > 0$, the statistical dimension $s_{\mu,\eps}$ is defined:
\begin{align*}
s_{\mu,\eps} \defeq \trace\left(\cK_\mu(\cK_\mu + \eps \cI_T)^{-1}\right),
\end{align*}
where $\cI_T$ is the identity operator on $L_2(T)$ and \(\trace\) is the trace of an operator.
\end{definition}
Refer to \cite{avron2019universal} for bounds on the statistical dimension of common kernels. For example, for an RBF kernel with lengthscale $\sigma^2$, $s_{\mu,\eps} \leq O(\sigma^2T \sqrt{\log(1/\eps)} + \log(1/\eps))$. For a sinc kernel with bandlimit $F$, $s_{\mu,\eps} = O(FT+ \log(1/\eps))$.

\cite{avron2019universal} prove that \problemref{fixed-kernel} can be solved with a number of samples depending on the statistical dimension $s_{\mu,\eps}$ as long as active samples are drawn from the following distribution over \([0,T]\):
\begin{definition}[Universal Sampling Distribution\footnote{
	The polynomial factors on \(\alpha\) can be tightened using some recent papers \cite{Erdelyi:2017,ChenPrice:2019}, but this would only tighten constants in \(\int_0^T \tilde\tau_\alpha(t) dt\), and hence only tighten constants in the sample complexity.
}]
\label{def:univ_dist}
For a parameter \(\alpha > 0\), let
\begin{align*}
	\tilde\tau_\alpha(t) \defeq \begin{cases}
		\frac{\alpha}{\min\{t,T-t\}} & t \in [T \frac1{\alpha^6}, T(1-\frac1{\alpha^6})] \\
		\frac{\alpha^6}{T}           & t \in [0, T\frac1{\alpha^6}] \cup [T(1-\frac1{\alpha^6}), T]
	\end{cases}
\end{align*}
Note that \(\int_0^T \tilde\tau_\alpha(t) dt = O(\alpha \log \alpha)\).
\end{definition}
Surprisingly, this distribution works for \emph{any} kernel PDF \(\mu\) and \(\eps > 0\), as long as \(\alpha \geq c s_{\mu,\eps}\) for some universal constant \(c > 0\). Specifically, \cite{avron2019universal} show that \(n = \Omega(s_{\mu,\eps} (\frac1\delta + \log(s_{\mu,\eps})))\) independent samples drawn from \([0,T]\) with probability proportional to \(\tilde\tau_\alpha(t)\) suffice to solve \problemref{fixed-kernel} with probability \((1-\delta)\).
The result relies on proving that \(\tilde\tau_\alpha\) forms an upper bound for the \textit{Ridge Leverage Function} of \(\cF_\mu^*\). Details are discussed in \appendixref{multiple-priors}, specifically \lemmaref{op_concentration} and \lemmaref{pairwise_lev_score_bound}.

\subsection{Spectral Mixture Kernels}

The core goal of this paper is to bound the sample complexity of learning kernel hyperparameters under adversarial noise.
While our techniques can apply to a wide variety of kernel classes, we illustrate their application with the Spectral Mixture (SM) kernel.
In particular, the SM Kernel has garnered interest in the Gaussian Process community for its ability to interpolate and extrapolate periodic structure very well \cite{wilson2013gaussian,DBLP:journals/corr/WilsonGNC13,DBLP:conf/aistats/YangWSS15,he2015state,DBLP:conf/nips/TobarBT15}.
However, hyperparameter tuning is also known to be difficult for SM kernels in practice \cite{DBLP:conf/icml/WilsonN15,benton2019function,DBLP:conf/icml/BuiHHLT16,DBLP:journals/jmlr/HensmanDS17,DBLP:conf/nips/WilsonDLX15}.
The SM Kernel is defined by having a PDF that is a symmetric mixture of Gaussians.

Formally, let \(\mu_{c,\sigma}(\xi)\defeq\frac{1}{\sqrt{2\pi\sigma^2}}e^{-\frac{(\xi-c)^2}{2\sigma^2}}\) denote a Gaussian PDF with mean \(c\) and lengthscale \(\sigma^2\).
Then, let \(\mu_{\vc,\vsigma,\vw}(\xi)\defeq\sum_{j=1}^q w_j \mu_{c_j,\sigma_j}(\xi)\) denote a mixture of \(q\) Gaussians with weights in \vw, means in \vc, and lengthscales in \vsigma.
The SM Kernel considers the special case of the mixture of Gaussians kernel when the PDF is symmetric: \(d\mu_{\vc,\vsigma,\vw}(-\xi)= d\mu_{\vc,\vsigma\vw}(\xi)\), making the kernel function real-valued:
\[
	k_{\vc,\vsigma,\vw}(s-t) =
	\sum_{j=1}^q w_j e^{-2\pi^2 (s-t)^2 \sigma_j^2} \cos(2\pi(s-t) c_j)
\]
All our results are stated for the Mixture of Gaussians kernel, so the SM kernel is handled implicitly.

\section{Technical Overview}
\label{sec:technical-overview}

At a high level, we are given a possibly infinite set of PDFs over frequencies \cU and want to find a specific PDF \(\tilde\mu\in\cU\) such the KRR interpolant using \(\tilde\mu\) is a good interpolant for the ground truth signal \(y(t)\).
We only get to observe \(y(t)\) through adversarially perturbed samples, and we get to pick those samples to lie anywhere in \([0,T]\).
Our main concern is bounding the number of samples needed to identify a near-optimal \(\tilde\mu\) and its associated interpolant \(\tilde y\).
We formally restate \problemref{interpolate} below:

\begin{repproblem}{interpolate}
Let \(y(t)\) be a signal we want to interpolate.
Let \(z(t)\) be an adversarial noise signal.
Let \cU be a (possibly infinite) set of kernel PDFs.
Let \(\hat\cU\subseteq\cU\) be the subset of PDF capable of representing \(y\) exactly\footnote{
	This is a technical nuance to handle the edge-case that \(y(t)\) might not be representable by all of the given PDFs.
	For instance, if \(y\) is a sinusoid with frequency 1, then a bandlimited \(\mu\) supported on frequencies \(2\) through \(4\) is incapable to of representing \(y\) exactly.
}.
That is, \(\hat\cU\defeq\{\mu\in\cU \ | \ \exists h\in L_2(\mu), y = \cF_\mu^* h\}\).
Fix regularization parameter \(\eps > 0\) and number of observations \(n\).
Observe \(y(t)+z(t)\) at any chosen times \(t_1,\ldots,t_n\).
Using any \(\tilde\mu\in\cU\), construct an interpolant \(\tilde y\) from our observations such that
\[
	\normof{y - \tilde y}_T^2 \leq C \cdot \Big(\normof{z}_T^2 + \eps \min_{\substack{\mu\in\hat\cU, \hspace{.1em}y=\cF_\mu^*h}}\normof{h}_\mu^2 \Big)
\]
\end{repproblem}

Note that, for any \(\mu\in\hat\cU\), we have defined \(\text{Energy}_\mu(y)\) to be \(\normof{h}_\mu^2\), where \(y=\cF_\mu^*h\).
That is, the energy of \(y\) under PDF \(\mu\) is the norm of the signal whose Inverse Fourier Transform is \(y\).
Intuitively, if it is difficult (requires a high energy signal) to represent \(y\) in \(L_2(\mu)\), then the energy of \(y\) is large.

To make our statistical approach clear, we start by presenting the exact time-sampling and interpolation schemes used in this paper.
We need two algorithms for our analysis: the first picks \(n\) times samples and builds \(Q\) different weighted kernel matrices (one for each of \(Q\) different given kernels); the second constructs a KRR interpolant for any given weighted kernel matrix.
Note that all kernel matrices are constructed \textit{using the exact same time samples}.
\begin{algorithm}[H]
	\caption{Time Point Sampling}
	\label{alg:time-sampling}
	{\bfseries input}: Kernel functions \(k_{\mu_1},\ldots k_{\mu_Q}\), non-negative function \(p(t)\) on \([0,T]\) with known integral \(P = \int_0^T p(t)dt\), number of samples \(n\).\\
	{\bfseries output}: Times \(t_1,\ldots, t_{n} \in [0,T]\), weights \(v_1, \ldots, v_n\), PSD matrices \(\mK_{\mu_1},\ldots, \mK_{\mu_Q} \in \bbC^{n\times n}\). \phantom{aaaaaaaaaaaaaaaaaaaaaaaaaaaaaaa}
	\vspace{-\baselineskip}
\begin{algorithmic}[1]
	\STATE Independently sample \(t_{1},\ldots, t_{n}\) from \([0,T]\) with probability proportional to \(p(t)\).
	\STATE For \(i\in \{1,\ldots, n\}\) set \(v_i \defeq  \sqrt{\frac{P}{n\cdot T\cdot p(t_i)}}\).
	\STATE For \(q\in \{1,\ldots, Q\}\) and \(i,j\in \{1,\ldots, n\}\) set \([\mK_{\mu_q}]_{i,j} \defeq v_iv_j \cdot k_{\mu_q}(t_i,t_j)\)
	\STATE {\bfseries return} \(t_1,\ldots,t_n\), \(v_1,\ldots,v_n\), \(\mK_{\mu_1},\ldots, \mK_{\mu_Q}\).
\end{algorithmic}
\end{algorithm}
Ultimately, we will take \(p(t)\) to be the universal sampling distribution \(\tau_\alpha(t)\) for some \(\alpha\), but state the sampling method for a general distribution.
For any particular \(\mK_\mu\), we can compute and evaluate the interpolant \(\tilde y\) as follows:
\begin{algorithm}[H]
	\caption{Computing the Interpolant}
	\label{alg:ridge-regression}
	{\bf input}: Time points \(t_1,\ldots, t_{n} \in [0,T]\), weights \(v_1, \ldots, v_n\), PSD matrix \(\mK_\mu \in \bbC^{n\times n}\), regularization parameter \(\eps > 0\).
	\\
	{\bf ouput}: Reconstructed function \(\tilde y\), represented implicitly
\begin{algorithmic}[1]
	\STATE Let \(\bar{\by} \in \bbC^n\) be the vector with \(\bar y_i = v_i \cdot [y(t_i) + z(t_i)]\)
	\STATE {\bf return} \(\tilde{\valpha}:= (\mK_\mu + \eps  \mI)^{-1}\bar{\by}\).
\end{algorithmic}
\end{algorithm}
For any \(t\) in \([0,T]\), we can evaluate \(\tilde{y}(t)\) by computing \(k_\mu(t_i,t)\) for all \(i\in 1,\ldots,n\) and returning \(\tilde y(t) = \sum_{i=1}^n \tilde\alpha_i \cdot k_\mu(t_i,t)\).

In order to start off the analysis, we show that solving a Fourier operator analogue to a Ridge Regression problem guarantees a solution to \problemref{interpolate}.
That is, we reduce the problem of finding a good interpolant to the problem of solving a specialized Ridge Regression problem.
However, this Ridge Regression problem involves an operator on \([0,T]\), and is not in terms of samples observed.
So, we then have to bound how many samples we need to observe for our samples to generalize well to the continuous \([0,T]\) domain, for all PDFs \(\mu\in\cU\).
This is the core technical challenge of this paper.

\begin{claim}
\label{clm:interpolate}
Let \(\tilde\mu\in\cU\) and \(\tilde g\in L_2(\mu)\) be near-optimal solutions to a continuous-time Fourier Fitting problem with ridge regularization:
\begin{align}
\label{ineq:interpolate-requirement}
	\normof{\cF_{\tilde\mu}^*\tilde g - (y+z)}_T^2 + \eps\normof{\tilde g}_{\tilde\mu}^2
	\leq C ~ \min_{\mu \in \cU} \ \min_{g\in L_2(\mu)} \left[\normof{\cF_\mu^* g - (y + z)}_T^2 + \eps\normof{g}_\mu^2\right]
	\nonumber
\end{align}
Let \(\hat\cU\subseteq\cU\) be the subset of PDFs that are able of representing \(y\) exactly.
Then,
\[
	\normof{y - \tilde y}_T^2
	\leq
	2(C+1)\normof{z}_T^2 + 2C\eps \min_{\substack{\mu\in\hat\cU,\hspace{.1em}y=\cF_\mu^*h}} \normof{h}_\mu^2
\]
\end{claim}
This claim is proven in \appendixref{interpolate}, and directly generalizes the proof of Claim 4 in \cite{avron2019universal}.
Our goal is now to find a \(\tilde\mu\) and \(\tilde g\) that approximately minimize \(\normof{\cF_\mu^* g - (y + z)}_T^2 + \eps\normof{g}_\mu^2\).
If we only had one \(\mu\) to consider, the prior work would be able to solve this with \(O(s_{\mu,\eps}(\frac1\delta + \log(s_{\mu,\eps})))\) many samples.
However, since our goal is to analyze hyperparameter tuning, we consider the cases with both exponentially large and infinitely large \cU.
In these cases, union bounds using prior work would yield exponentially large and unbounded sample complexities, respectively.
In order to avoid this, we form an epsilon-net style argument.
The argument follows in two steps:
\begin{enumerate}[leftmargin=*]
	\item Sampling Time with Finitely Many PDFs:
	Assume that \(\cU\) is finite.
	Let \(s_{\max,\eps}\) be the largest statistical dimension found in \cU.
	Then we prove that \(O(s_{\max,\eps} \log(\frac{s_{\max,\eps}}{\delta} \cdot \abs{\cU}))\) observations suffice to recover a near-optimal \((\tilde\mu,\tilde g)\) pair.
	We emphasize the logarithmic dependence on \(\abs\cU\), since this will allow us to consider exponentially large sets in the next step.
	\item Discretization of Kernel Hyperparameters:
	Assume that \(\cU\) is the set of Gaussian Mixture PDFs with \(q\) Gaussians, taking means in \([-W,W]\), lengthscales in \([m,M]\), and weights in \([0,1]\).
	Then we create a \textit{finite} set of Gaussian Mixture PDFs \(\tilde\cU\) such that the best \((\tilde\mu,\tilde g)\) pair on \(\tilde\cU\) is nearly optimal on all of \cU.
	In particular, we find \(\abs{\tilde\cU} = O((\frac Wm \log (\frac Mm))^q)\).
\end{enumerate}
Our result from the first bullet point allows us to handle the exponentially large set \(\tilde\cU\) created in the second bullet point.
After combining these results and noting that \(s_{\max,\eps} = \tilde O(qMT)\), we find that \(\tilde O(q^2 MT\log (\frac Wm))\) time samples suffice to identify a near-optimal SM kernel's hyperparameters.
The rest of this paper breaks down and explains these two theoretical results in detail.

\section{Sampling Time with Finitely Many PDFs}
\label{sec:multiple-priors}

In this section we assume that the given set of  PDFs \cU is finite, and let \(Q \defeq \abs{\cU}\).
Let \(\tilde y\) and \(\tilde\mu\in\cU\) be the KRR interpolant and associated PDF that minimize our sample ridge regression cost.
We then prove that \(\tilde y\) describes a nearly-optimal interpolant that satisfies the requirement of \claimref{interpolate}, so long as we take sufficient samples from the Universal Sampling Distribution (\defref{univ_dist}).
In particular, if \(s_{\max,\eps}\) is the largest statistical dimension found in \cU, then we require \(O(s_{\max,\eps} \log(\frac{s_{\max,\eps}}{\delta} \cdot Q))\) samples.
We formally state this first core technical result:

\begin{theorem}
\label{thm:multiple-priors}
Let \(\cU\) be a finite set of PDFs.
Let \(s_{\max,\eps}\) be the maximum statistical dimension in \cU.
Let \algorithmref{time-sampling} output observation times \(t_1,\ldots,t_n\), weights \(v_1,\ldots,v_n\), and weighted Kernel Matrices \(\mK_{\mu_1},\ldots,\mK_{\mu_Q}\).
Let \(\bar\vy\) be the observed response vector.
Let \(\tilde\mu,\tilde\valpha\) solve the ridge regression problem:
\begin{align}
\label{eq:time-discrete-fourier}
	\tilde\mu, \tilde\valpha
	\defeq \argmin_{\mu\in\cU, \valpha\in\bbR^n}
	\normof{\mK_\mu\valpha - \bar\vy}_2^2 + \eps\valpha^\intercal\mK_\mu\valpha
\end{align}
Define the Fourier domain version of the interpolant\footnote{This parametrization simply ensures that \(\tilde y(t) = [\cF_\mu^* \tilde g](t)\)}: \(\tilde g(\xi) \defeq \sum_{j=1}^n v_j \tilde \alpha_j e^{-2\pi i \xi t_j}\).
If \(n = \Omega(s_{\max,\eps} \log(\frac{s_{\max,\eps}}{\delta} \cdot Q))\), then with probability \(1-\delta\) we have
\begin{align*}
	\normof{\cF_{\tilde\mu}^*\tilde g - (y+z)}_T^2 + \eps\normof{\tilde g}_{\tilde\mu}^2
	\leq
	(9+\nicefrac8\delta) ~ \min_{\mu\in\cU} \min_{g\in L_2(\mu)} \normof{\cF_{\mu}^* g - (y+z)}_T^2 + \eps\normof{g}_{\mu}^2
\end{align*}
\end{theorem}

\theoremref{multiple-priors} is proven in \appendixref{multiple-priors}, with a simplified and more approachable proof for the matrix case in \appendixref{multiple-designs}. The approach taken is similar to that used for the two-sided approximate regression problem addressed in Lemma 5.7 of \cite{EldarLiMusco:2020}.
Intuitively, \theoremref{multiple-priors} states that despite having \emph{adversarial noise}, choosing from a large family of kernels during hyperparameter tuning does not sharply increase the sample complexity of fitting \(y(t)\).
In other words, \theoremref{multiple-priors} states that \(\Omega(s_{\max,\eps} \log(\frac{s_{\max,\eps}}\delta \cdot Q))\) samples guarantees a solution to \problemref{interpolate} when \(\cU\) is finite.

In prior work, \cite{avron2019universal} proves that \(\Omega(s_{\mu,\eps} \log(s_{\mu,\eps} + \frac1\delta))\) samples guarantees a solution to \problemref{fixed-kernel}, and that this bound is tight for many common kernels.
Since \problemref{interpolate} reduces to \problemref{fixed-kernel} when \(Q=1\), the sample complexity in \theoremref{multiple-priors} must be tight up to logarithmic factors.
Additionally, note that union bounding this result from \cite{avron2019universal} over the \(Q\) kernels would yield a sample complexity linear in \(Q\), instead of the logarithmic rate we prove.
This logarithmic rate is important, since the next section will take \(Q\) to be exponentially large.

It remains unclear if the dependence on \(\frac1\delta\) in the approximation error is necessary if we want a logarithmic sample complexity dependence on \(Q\).
%The standard randomized numerical linear algebra approach requires \(O(Q)\) many matrix-multiplication claims, which then incurs a \(O(\frac Q\delta)\) sample complexity.
This is an interesting open problem even in the case of least squares regression, where we choose one of \(Q\) different design matrices.

So, \theoremref{multiple-priors} tells us that we can choose from a finite set of kernels $Q$ with only a logarithmic $\log(Q)$ overhead in sample complexity. However practitioners do not consider finite sets of kernels, but rather kernel classes like the SM Kernel, which are parameterized by several \textit{continuous real-valued parameters}.
So, we cannot directly apply \theoremref{multiple-priors} to SM Kernel fitting; one more step is needed.

\section{Discretization of Spectral Mixture Hyperparameters}
\label{sec:discrete-hyperparameters}

We now return to the original goal of hyperparameter tuning for kernels.
At a high level, we expect that a sufficiently small change to a kernel's hyperparameters should not substantially impact the quality of the kernel as an interpolant.
So, instead of considering the continuous range of all hyperparameters, we create a \emph{finite} net of hyperparameters \(\tilde\cU\).
In particular, any selection of hyperparameters \(\hat\mu\in\cU\) has a corresponding selection of hyperparameters \(\tilde\mu\) that lies in the net \(\tilde\cU\).
Since we design \(\tilde\mu\) to be sufficiently similar to \(\hat\mu\), we can then prove that \(\hat\mu\) cannot achieve a much smaller error than \(\tilde\mu\).
Intuitively, we can think \(\tilde\cU\) as being a discretization of the full continuous set of hyperparameters \(\cU\).

Then, once we have constructed the discretization \(\tilde\cU\), we can use \theoremref{multiple-priors} to prove that \(n=O(s_{\max,\eps}\log(\frac{s_{\max,\eps}}{\delta} \cdot \abs{\tilde\cU}))\) observations suffice to interpolate \(y\) with a near-optimal choice of hyperparameters.
Since \theoremref{multiple-priors} admits a logarithmic dependence on the size of our net \(\abs{\tilde\cU}\), we can create an exponentially large net while achieving polynomial sample complexity bounds.

This broad principle of discretization can easily apply to many kernels; for demonstration purposes, we only consider the SM Kernel in this work.
To bound the sample complexity of other kernels, it would suffice to form a bound like \theoremref{discrete-hyperparameters} below.
Here we assume that \cU is the set of Gaussian Mixture hyperparameters, mixing \(q\) Gaussians with means in \([-W,W]\), lengthscales in \([m,M]\), and weights in \([0,1]\).

\begin{theorem}
\label{thm:discrete-hyperparameters}
Fix the constants \(W,m,M\) as described above.
Define the discretization set for means as
\[
	\cC \defeq \{-W, -W+m, -W+2m, \ldots, (k-2)m, W\}
\]
and the discretization set for lengthscales as
\[
	\cS \defeq \{m, 2m, 4m, 8m, \ldots, 2^{\ell-3}m, M, 2M\}
\]
where \(k = \floor{\frac{2W}{m}} = \abs\cC\) and \(\ell = \floor{\log_2(M/m)}+1 = \abs\cS\).
Then we have
\begin{align*}
	\min_{\substack{g\in L_2(\mu_{\vc,\vsigma,\mathbf1}):\\\vc\in\cC^q\\\vsigma\in\cS^q}}
	\normof{\cF_{\vc,\vsigma,\mathbf{1}}^*\tilde g - (y+z)}_T^2 + \eps\normof{\tilde g}_{\vc,\vsigma,\mathbf 1}^2
	\hspace{0cm}\leq 8 \cdot
	\min_{\substack{g\in L_2(\mu_{\vc,\vsigma,\vw}):\\\vc\in[-W,W]^q\\\vsigma\in[m,M]^q\\\vw\in[0,1]^q}}
	\normof{\cF_{\vc,\vsigma,\vw}^*\tilde g - (y+z)}_T^2 + \eps\normof{\tilde g}_{\vc,\vsigma,\vw}^2
\end{align*}
where \(\mathbf 1\) is the all-ones vector.
\end{theorem}
\theoremref{discrete-hyperparameters} is proven in \appendixref{discrete-hyperparameters}. Note that the left hand side takes \vw equal to the all ones vector, denoted \(\mathbf 1\).
This is without loss of generality, since increasing the scale of the kernel matrix (i.e. increasing any weight \(w_j\)) monotonically increases the statistical dimension and decreases the regularized mean squared error.
So, as long as we have enough samples to satisfy the statistical dimension requirement when \(\vw = \mathbf{1}\), we should take \(\vw\) to be all-ones without loss of generality.

Intuitively, \theoremref{discrete-hyperparameters} reduces the search space for SM kernel hyperparameters down to a finite set of kernels.
This allows us to apply \theoremref{multiple-priors} to general SM Kernel fitting.
Using \claimref{interpolate} as well, we form the following conclusion on the statistical cost of learning SM Kernel hyperparameters:
\begin{corollary}
\label{corol:sm-bounds}
Suppose we want to fit a signal using a SM Kernel with \(q\) Gaussians whose means lie in \([0,W]\), lengthscales lie in \([m,M]\), and weights lie in \([0,1]\).
Then, with probability 0.99, \(n=\tilde O(q^2MT\log(\frac Wm))\) time samples drawn from the Universal Sampling Distribution suffice to have the KRR interpolant \(\tilde y\) give 
\[
	\normof{y - \tilde y}_T^2
	\leq
	C\cdot \Big(\normof{z}_T^2 + \eps \min_{\substack{\mu\in\hat\cU,\hspace{.1em}y=\cF_\mu^*h}} \normof{h}_\mu^2\Big)
\]
where \(\hat\cU\) is the set of valid SM kernels capable of representing \(y\).
\end{corollary}
\begin{proof}
\cite{avron2019universal} shows that the statistical dimension of a mixture of \(q\) Gaussians is at most \(s_{\max,\eps} \leq q \cdot (MT\sqrt{\log(1/\eps)} + \log(1/\eps))\).
\theoremref{discrete-hyperparameters} tell us that we need to consider \(Q=O((\frac Wm \log(\frac Mm))^q)\) specific prior hyperparameters.
Then, \theoremref{multiple-priors} tells us that \(O(s_{\max,\eps} \log(\frac{s_{\max,\eps}}{\delta} \cdot Q))\) samples suffice to satisfy the precondition for \claimref{interpolate}, giving us a sample complexity of
\begin{align*}
	&\phantom{==} O\bigg( q^2 \cdot
			(MT\sqrt{\log\nicefrac1\eps} + \log\nicefrac1\eps) \ \cdot \log\left(
						\frac{MT\sqrt{\log\nicefrac1\eps} + \log\nicefrac1\eps}{\delta}
						\cdot
						\frac Wm \log \left(\frac Mm\right)
					\right)
		\Bigg) \\
	&= \tilde O\left(q^2 MT\log\left(\frac Wm\right)\right)
\end{align*}
\end{proof}

Note that the \(\tilde O\) notation hides a logarithmic dependence on \(\frac1\eps\) and a sublogarithmic dependence on \(\frac{M}{m}\).
Note that \cite{avron2019universal} proves that a single Gaussian kernel with lengthscale \(M\) would already require \(\tilde O(MT)\) samples, so hyperparameter tuning for a single Gaussian only increases the sample complexity by logarithmic factors.
However, when we consider multiple Gaussians, our analysis does introduce an extra factor of \(q\) beyond statistical dimension \(s_{\max,\eps} = \tilde O(qMT)\).

\section{Conclusion}
Despite how useful SM Kernels are \cite{wilson2013gaussian,DBLP:journals/corr/WilsonGNC13,DBLP:conf/aistats/YangWSS15,he2015state,DBLP:conf/nips/TobarBT15}, practitioners find that tuning SM Kernels is hard in practice \cite{DBLP:conf/icml/WilsonN15,benton2019function,DBLP:conf/icml/BuiHHLT16,DBLP:journals/jmlr/HensmanDS17,DBLP:conf/nips/WilsonDLX15}.
A practitioner could consider two reasons why it is hard to fit the SM Kernel: either they have too little data to information-theoretically find a good model, or their algorithms fail to find such a model despite having enough information.
Our final result, \corolref{sm-bounds}, shows that the statistical complexity of learning the SM kernel's hyperparameters is not too large, even against adversarial noise.
A natural conclusion is that practitioners should likely place effort in finding more effective algorithms.

We see several interesting potential future directions for this work.
First, this paper focuses its applications to the Spectral Mixture kernel.
Other popular kernels like the Matern, sinc, and Rational Quadratic kernels can also be analyzed under our framework.
Further, we provide statistical bounds for finding optimal hyperparameters by designing a discrete optimization problem over exponentially many PDFs, but we do not provide any polynomial time algorithm to solve this problem.
Lastly, we would like to know if the dependence on \(\frac1\delta\) in the approximation error of \theoremref{multiple-priors} is neccessary if we want a logarithmic dependence on \(Q\) in the sample complexity.

\section*{Acknowledgements}
The authors would like to thank Xue Chen for valuable discussion in the early stages of this work.

\bibliographystyle{alpha}
\bibliography{local}

\newpage

\appendix

\section{Interpolation Gaurantees}
\label{app:interpolate}
\begin{repclaim}{interpolate}
Let \(\tilde\mu\in\cU\) and \(\tilde g\in L_2(\mu)\) be near-optimal solutions to a continuous time Fourier Fitting problem with ridge regression:
\[
	\normof{\cF_{\tilde\mu}^*\tilde g - (y+z)}_T^2 + \eps\normof{\tilde g}_{\tilde\mu}^2
	\leq
	C ~
	\min_{\mu \in \cU} \ \min_{g\in L_2(\mu)}
	\left[\normof{\cF_\mu^* g - (y + z)}_T^2 + \eps\normof{g}_\mu^2\right]
\]
Let \(\hat\cU\subseteq\cU\) be the subset of PDFs that are able of representing \(y\) exactly.
Then, letting \(\tilde y = \cF_{\tilde \mu}^* \tilde g\),
\[
	\normof{y - \tilde y}_T^2
	\leq
	2(C+1)\normof{z}_T^2 + 2C\eps \min_{\substack{\mu\in\hat\cU\\y=\cF_\mu^*h}} \normof{h}_\mu^2
\]
\end{repclaim}

\begin{proof}
	Letting \(y=\cF_\mu^* h_\mu\) for all \(\mu\in\hat\cU\), we know that
\begin{align*}
	\min_{\mu \in \cU} \min_{g\in L_2(\mu)} \left[\normof{\cF_\mu^* g - (y + z)}_T^2 + \eps\normof{g}_\mu^2\right]
	&\leq \min_{\mu \in \hat\cU} \left[\normof{\cF_\mu^* h_\mu - (y + z)}_T^2 + \eps\normof{h_\mu}_\mu^2\right] \\
	&= \min_{\mu \in \hat\cU} \left[\normof{z}_T^2 + \eps\normof{h_\mu}_\mu^2\right] \\
	&= \normof{z}_T^2 + \eps \min_{\mu \in \hat\cU}\normof{h_\mu}_\mu^2 \\
\intertext{So, using our pair \(\tilde \mu, \tilde g\), we have}
	\normof{\cF_{\tilde \mu}^* \tilde g - (y + z)}_T^2 + \eps\normof{\tilde g}_{\tilde \mu}^2
	&\leq C\normof{z}_T^2 + C\eps \min_{\mu \in \hat\cU}\normof{h_\mu}_\mu^2
\intertext{Next, by the triangle inequality, and recalling that \(\tilde y = \cF_{\tilde\mu}^*\tilde g\),}
	\normof{\tilde y - y}_T &- \normof{z}_T \leq \normof{\cF_{\tilde \mu}^*\tilde g - (y + z)}_T \\
	\normof{\tilde y - y}_T &\leq \normof{z}_T + \sqrt{C\normof{z}_T^2 + C\eps \min_{\mu \in \hat\cU}\normof{h_\mu}_\mu^2} \\
	\normof{\tilde y - y}_T^2 &\leq 2(C+1)\normof{z}_T^2 + 2C \eps \min_{\mu\in\hat\cU} \normof{h_\mu}_\mu^2 \\
\end{align*}
where the last line uses the AM-GM inequality to bound
\[
	2 \cdot \normof{z}_T \cdot \sqrt{C\normof{z}_T^2 + C\eps \min_{\mu \in \hat\cU}\normof{h_\mu}_\mu^2}
	\leq
	\normof{z}_T^2 + C\normof{z}_T^2 + C\eps \min_{\mu \in \hat\cU}\normof{h_\mu}_\mu^2
\]
\end{proof}

\section{Spectral Mixture Bounds}
\label{app:discrete-hyperparameters}
We are allowed to use \(c\in[0,W]\), \(\sigma\in[m,M]\), and \(w\in[0,1]\), where \(0 < m < M\), and want to have at most \(O(1)\) error from our discretization.
We start by showing that without loss of generality, we should always take \(w\) to be the all-ones vector.
\begin{lemma}
\label{lem:operator-error-psd-order}
Let \(\mu_1\) and \(\mu_2\) be associated with kernel operators \(\cK_{\mu_1}\) and \(\cK_{\mu_2}\) such that \(\cK_{\mu_1} \preceq \cK_{\mu_2}\).
Then,
\[
	\min_{g\in L_2(\mu_1)}
	\normof{\cF_{\mu_1}^* g - \bar y}_T^2 + \eps \normof{g}_{\mu_1}^2
	\leq
	\min_{g\in L_2(\mu_2)}
	\normof{\cF_{\mu_2}^* g - \bar y}_T^2 + \eps \normof{g}_{\mu_2}^2
\]
\end{lemma}
\begin{proof}
For now, consider a arbitrary \(\mu\).
Note from Lemma 38 of \cite{avron2019universal}, we know that the minimizer of 
\[
	\min_{g\in L_2(\mu)}
	\normof{\cF_{\mu}^* g - y}_T^2 + \eps \normof{g}_{\mu}^2
\]
has the form \(\hat g = \cF_\mu (\cK_\mu + \eps\cI_T)^{-1}\bar y\).
Then, we can write
\begin{align*}
	\cF_\mu^* \hat g
	&= \cK_\mu(\cK_\mu+\eps\cI_T)^{-1}\bar y \\
	\normof{\cF_\mu^* - \bar y}_T^2
	&= \langle \cK_\mu(\cK_\mu+\eps\cI_T)^{-1}\bar y - \bar y, \cK_\mu(\cK_\mu+\eps\cI_T)^{-1}\bar y - \bar y\rangle_T \\
	&= \normof{\bar y}_T^2 - 2 \langle \bar y, \cK_\mu(\cK_\mu+\eps\cI_T)^{-1}\bar y\rangle_T + \langle \cK_\mu(\cK_\mu+\eps\cI_T)^{-1}\bar y, \cK_\mu(\cK_\mu+\eps\cI_T)^{-1}\bar y\rangle_T \\
	\normof{\hat g}_\mu^2 &= \langle \cK_\mu(\cK_\mu+\eps\cI_T)^{-1}\bar y, \cK_\mu(\cK_\mu+\eps\cI_T)^{-1}\bar y\rangle_T
\end{align*}
Noting that all the inner products on the last two lines share the same right hand side, we find that the value of the true minimizer is
\begin{align*}
	\normof{\cF_\mu^* - \bar y}_T^2 + \eps\normof{\hat g}_\mu^2
	&= \normof{\bar y}_T^2 + \langle -2\bar y + \cK_\mu(\cK_\mu+\eps\cI_T)^{-1}\bar y + \eps(\cK+\eps\cI_T)^{-1}\bar y, \cK_\mu(\cK_\mu+\eps\cI_T)^{-1}\bar y\rangle_T \\
	&= \normof{\bar y}_T^2 + \langle (-2(\cK_\mu+\eps\cI_T) + \cK_\mu + \eps\cI_T) (\cK_\mu+\eps\cI_T)^{-1}\bar y, \cK_\mu(\cK_\mu+\eps\cI_T)^{-1}\bar y\rangle_T \\
	&= \normof{\bar y}_T^2 + \langle -1 \cdot (\cK_\mu+\eps\cI_T) \cdot (\cK_\mu+\eps\cI_T)^{-1}\bar y, \cK_\mu(\cK_\mu+\eps\cI_T)^{-1}\bar y\rangle_T \\
	&= \normof{\bar y}_T^2 - \langle \bar y, \cK_\mu(\cK_\mu+\eps\cI_T)^{-1}\bar y\rangle_T \\
	&= \langle \bar y, \cI_T - \cK_\mu(\cK_\mu+\eps\cI_T)^{-1}\bar y\rangle_T \numberthis \label{eq:min-operator-error}
\end{align*}
Then, since we know that the kernel operator \(\cK_\mu \succeq 0\), we conclude that \(\cI_T - \cK_\mu(\cK_\mu+\eps\cI_T)^{-1} \succeq 0\) for all kernel operators \(\cK_\mu\).
Additionally, note that \equationref{min-operator-error} is in the analogous form to \(\vx^\intercal\mA\vx\) for matrices.
In particular, if we decrease the semidefinite order of \(\cI_T - \cK_\mu(\cK_\mu+\eps\cI_T)^{-1}\), then we decrease the overall minimum value for all signals \(\bar y\).
Since \(\cK_{\mu_1} \preceq \cK_{\mu_2}\), we know that \(\cK_{\mu_1}(\cK_{\mu_1}+\eps\cI_T)^{-1} \preceq \cK_{\mu_2}(\cK_{\mu_2}+\eps\cI_T)^{-1}\), and hence
\[
	\langle \bar y, \cI_T - \cK_{\mu_2}(\cK_{\mu_2}+\eps\cI_T)^{-1}\bar y\rangle_T
	\leq
	\langle \bar y, \cI_T - \cK_{\mu_1}(\cK_{\mu_1}+\eps\cI_T)^{-1}\bar y\rangle_T
\]
Or, equivalently,
\[
	\min_{g\in L_2(\mu_1)}
	\normof{\cF_{\mu_1}^* g - \bar y}_T^2 + \eps \normof{g}_{\mu_1}^2
	\leq
	\min_{g\in L_2(\mu_2)}
	\normof{\cF_{\mu_2}^* g - \bar y}_T^2 + \eps \normof{g}_{\mu_2}^2
\]
\end{proof}
We now show why this tell us to pick the all-ones vector for SM Kernels.
In the following, when \vw is the all-ones vector, we drop \vw from the subscripts (i.e. \(\mu_{\vw,\vsigma}\defeq\mu_{\vw,\vsigma,\mathbf1}\) where \(\mathbf1\) is the all-ones vector).
\begin{corollary}
Let \(\mu_{\vc,\vsigma,\vw}(\xi)\) be a spectral mixture PDF with weights \(\vw\in[0,1]^q\).
Then, the spectral mixture with the same means and lengthscales \(\mu_{\vc,\vsigma}(\xi)\) achieves uniquely less error:
\[
	\min_{g\in L_2(\mu_{\vc,\vsigma,\vw})}
	\normof{\cF_{\vc,\vsigma,\vw}^* g - \bar y}_T^2 + \eps \normof{g}_{\vc,\vsigma,\vw}^2
	\leq
	\min_{g\in L_2(\mu_{\vc,\vsigma})}
	\normof{\cF_{\vc,\vsigma}^* g - \bar y}_T^2 + \eps \normof{g}_{\vc,\vsigma}^2	
\]
\end{corollary}
\begin{proof}
Note that the Kernel operator associated with \(\mu_{\vc,\vsigma\vw}\) is \(\cK_{\vc,\vsigma,\vw} = \sum_{j=1}^q w_j \cK_{c_j, \sigma_j}\).
Since \(w_j \leq 1\), we find that \(\cK_{\vc,\vsigma,\vw} \preceq \sum_{j=1}^q \cK_{c_j,\sigma_j} = \cK_{\vc,\vsigma}\), the kernel operator associated with the all ones weight vector.
So, by \lemmaref{operator-error-psd-order}, we complete the proof.
\end{proof}

With this reduction in place, we move onto consider the means and lengthscales of our kernel.
We are allowed to use means \(c\in[0,W]\) and lengthscales \(\sigma\in[m,M]\), where \(0 < m < M\), and want to have at most \(O(1)\) error from our discretization.
We achieve this with additive mean step sizes and multiplicative lengthscale step sizes,:
\begin{align*}
	\cC &= \{0, \rho m, 2\rho m, \ldots, (k-2)\rho m, W\} \\
	\cS &= \{m, (1+\gamma) m, (1+\gamma)^2m, \ldots, (1+\gamma)^{\ell-3}m, M, (1+\gamma)M\} \\
\end{align*}
Note that the step sizes for both the means and lengthscales are left in terms of the minimum lengthscale.
The set \cC guarantees that any \(\hat c\in[0,W]\) has \(\tilde c\in\cC\) such that
\begin{align}
\label{eq:sm-mean-net}
	\abs{\tilde c - \hat c} \leq \rho m
\end{align}
Additionally, set \cS guarantees that any \(\hat\sigma \in[m,M]\) has \(\tilde\sigma\in\cS\) such that
\begin{align}
\label{eq:sm-lengthscale-net}
	\frac{\tilde\sigma}{(1+\gamma)^2} \leq \hat\sigma \leq \frac{\tilde\sigma}{1+\gamma} < \tilde\sigma
\end{align}
Notably, we do not allow \(\hat\sigma\) to be arbitrarily close to \(\tilde\sigma\), but instead guarantee a multiplicative gap between the two.
This is why the maximum value of \cS is greater than \(M\).
There are \(k = \floor{\frac{W}{\rho m}}\) means in \cC and \(\ell = \floor{\frac{\ln(2M/m)}{\ln(1+\gamma)}}\) lengthscales in \cS.
We now discretize a SM kernel of \(q\) Gaussian modes by rounding to means in \cC and lengthscales in \cS:
\begin{lemma}
\begin{align}
\label{ineq:sm-net}
	\min_{\substack{\vc\in\cC^q \\ \vsigma\in\cS^q}} \min_{g\in L_2(\mu_{\vc,\vsigma})}
	\normof{\cF_{\vc,\vsigma}^* g - \bar y}_T^2 + \eps \normof{g}_{\vc,\vsigma}^2
	\leq
	C \ \min_{\substack{\vc\in[0,W]^q\\\vsigma\in[m,M]^q}} \min_{g\in L_2(\mu_{\vc,\vsigma})}
	\normof{\cF_{\vc,\vsigma}^* g - \bar y}_T^2 + \eps \normof{g}_{\vc,\vsigma}^2
\end{align}
Where \(C = (1+\gamma)^2 \exp(\frac{\rho^2}2 \cdot \frac{1}{1-\frac{1}{(1+\gamma)^2}} )\).
\end{lemma}
If we want a factor of 3 error, we can take \(\gamma = \rho = 0.5\), so that \(C \approx 2.8178 < 3\).
This makes \(\abs{\cC}=O(\frac Wm)\) and \(\abs{\cS} = O(\log(M/m))\), so that the discretize space of \(q\) SM kernels has \(O((\frac Wm \log(\frac Mm))^q)\) choices of hyperparameter to consider.
\begin{proof}
Let \(\hat\vc,\hat\vsigma,\) and \(\hat g\) be the minimizers of the right hand side of \inequalityref{sm-net}.
Let \(p(\xi;c,\sigma) \defeq \frac{1}{\sqrt{2\pi\sigma^2}} e^{-\frac{(\xi-c)^2}{2\sigma^2}}\) be the Gaussian PDF with mean \(c\) and lengthscale \(\sigma^2\).
Further, let \(p(\xi;\vc,\vsigma) \defeq \sum_{j=1}^q p(\xi,c_j,\sigma_j)\) be the sum of the Gaussians described in \vc and \vsigma.
This allows us to write \(d\mu_{\vc,\vsigma}(\xi) = p(\xi;\vc,\vsigma) d\xi\).

Let \(\tilde\vc\) and \(\tilde\vsigma\) be the discretizations of \(\hat\vc\) and \(\hat\vsigma\) using the schemes from \equationref{sm-mean-net} and \equationref{sm-lengthscale-net}.
Let \(\tilde g\) be the following rounding of \(\hat g\):
\[
	\tilde g(\xi)
	\defeq
	\hat g(\xi) \cdot \frac{p(\xi;\hat\vc,\hat\vsigma)}{p(\xi;\tilde\vc,\tilde\vsigma)}
\]
Then, this particular rounding implies that the Inverse Fourier Transform of \(\tilde g\) preserves the Inverse Fourier Transform of \(\hat g\):
\begin{align*}
	[\cF_{\tilde\vc,\tilde\vsigma}^* \tilde g](t)
	&= \int_\bbR \tilde g(\xi) e^{2\pi i \xi t} d\mu_{\tilde\vc,\tilde\vsigma}(\xi) \\
	&= \int_\bbR \hat g(\xi) \cdot \frac{p(\xi;\hat\vc,\hat\vsigma)}{p(\xi;\tilde\vc,\tilde\vsigma)} \cdot e^{2\pi i \xi t} \cdot p(\xi,\tilde\vc,\tilde\vsigma) \cdot d\xi \\
	&= \int_\bbR \hat g(\xi) \cdot p(\xi;\hat\vc,\hat\vsigma) \cdot e^{2\pi i \xi t} \cdot d\xi \\
	&= \int_\bbR \hat g(\xi) e^{2\pi i \xi t} d\mu_{\hat\vc,\hat\vsigma}(\xi) \\
	&= [\cF_{\hat\vc,\hat\vsigma}^* \hat g](t)
\end{align*}
So, we immediately know that \(\normof{\cF_{\tilde\vc,\tilde\vsigma}^* \tilde g - \bar y}_T^2 = \normof{\cF_{\hat\vc,\hat\vsigma}^* \hat g - \bar y}_T^2\).
All we need to do now is bound the power of \(\tilde g\) with respect to \(\tilde\vc\) and \(\tilde\vsigma\):
\begin{align*}
	\normof{\tilde g}_{\tilde\vc,\tilde\vsigma}^2
	&= \int_\bbR \abs{\tilde g(\xi)}^2 d\mu_{\tilde\vc,\tilde\vsigma}(\xi) \\
	&= \int_\bbR \abs{\hat g(\xi)}^2 \left(\frac{p(\xi;\hat\vc,\hat\vsigma)}{p(\xi;\tilde\vc,\tilde\vsigma)}\right)^2 p(\xi;\tilde\vc,\tilde\vsigma) ~ d\xi \\
	&= \int_\bbR \abs{\hat g(\xi)}^2 \frac{p(\xi;\hat\vc,\hat\vsigma)}{p(\xi;\tilde\vc,\tilde\vsigma)} \ p(\xi;\hat\vc,\hat\vsigma) ~ d\xi \\
	&\leq C \int_\bbR \abs{\hat g(\xi)}^2 p(\xi;\hat\vc,\hat\vsigma) ~ d\xi \\
	&= C \normof{\hat g}_{\hat\vc,\hat\vsigma}^2
\end{align*}
The inequality uses the fact that \(\frac{p(\xi;\vc,\hat\vsigma)}{p(\xi;\vc,\tilde\vsigma)} \leq C\) for all \(\xi\), proven below.

First, we bound the ratio \(\frac{p(\xi;\hat c_1,\hat \sigma_1)}{p(\xi;\tilde c_1,\tilde\sigma_1)}\).
Note that
\[
	\frac{p(\xi;\hat c_1,\hat \sigma_1)}{p(\xi;\tilde c_1,\tilde\sigma_1)} = \frac{\tilde\sigma_1}{\hat\sigma_1} \exp\left(\frac{(\xi-\tilde c_1)^2}{2\tilde\sigma_1^2} - \frac{(\xi-\hat c_1)^2}{2\hat\sigma_1^2}\right)
\]
With some calculus, we can show that the maximum of the right hand side occurs when \(\xi = \frac{\hat\sigma_1^2\tilde c_1 - \tilde\sigma_1^2\hat c_1}{\hat\sigma_1^2 - \tilde\sigma_1^2}\) and attains hence maximum value
\[
	\frac{p(\xi;\hat c_1,\hat \sigma_1)}{p(\xi;\tilde c_1,\tilde\sigma_1)}
	\leq
	\frac{\tilde\sigma_1}{\hat\sigma_1}
	\exp\left(\frac12 \cdot \frac{(\tilde c_1 - \hat c_1)^2}{\tilde\sigma_1^2 - \hat\sigma_1^2}\right)
\]
We can then use our rounding schemes from \equationref{sm-mean-net} and \equationref{sm-lengthscale-net} to say
\begin{itemize}
 	\item
 	\(\frac{\tilde\sigma_1}{\hat\sigma_1} \leq \frac{~\tilde\sigma_1}{\frac{\tilde\sigma_1}{(1+\gamma)^2}} = (1+\gamma)^2\)
 	\item
 	\(\tilde\sigma_1^2 - \hat\sigma_1^2 \geq \tilde\sigma_1^2 - \frac{\tilde\sigma_1^2}{(1+\gamma)^2} = \tilde\sigma_1^2 (1-\frac{1}{(1+\gamma)^2}) \geq m^2 (1-\frac{1}{(1+\gamma)^2})\)
 	\item
 	\((\tilde c_1 - \hat c_1)^2 \leq \rho^2 m^2\)
\end{itemize}
With these three bounds, we conclude
\[
	\frac{p(\xi;\hat c_1,\hat \sigma_1)}{p(\xi;\tilde c_1,\tilde\sigma_1)}
	\leq
	(1+\gamma)^2 \exp(\frac{\rho^2}2 \cdot \frac{1}{1-\frac{1}{(1+\gamma)^2}}) = C
\]

Finally, we complete the proof by noting
\begin{align*}
	\frac{p(\xi;\hat\vc,\hat\vsigma)}{p(\xi;\hat\vc,\tilde\vsigma)}
	&= \frac{\sum_{j=1}^n p(\xi;\hat c_j,\hat\sigma_j)}{p(\xi;\tilde\vc,\tilde\vsigma)} \\
	&\leq \frac{\sum_{j=1}^n C \cdot p(\xi;\tilde c_j,\tilde\sigma_j)}{p(\xi;\tilde\vc,\tilde\vsigma)} \\
	&= C \cdot \frac{\sum_{j=1}^n p(\xi;\tilde c_j,\tilde\sigma_j)}{p(\xi;\tilde\vc,\tilde\vsigma)} \\
	&= C \cdot \frac{p(\xi;\tilde\vc,\tilde\vsigma)}{p(\xi;\tilde\vc,\tilde\vsigma)} \\
	&= C
\end{align*}
\end{proof}

\section{Multiple Prior Subsampling Bounds}
\label{app:multiple-prior-proofs}
\subsection{Proof for the Matrix Case}
\label{app:multiple-designs}

First, we introduce the matrix version of the ridge leverage function, first introduced in \cite{alaoui2015fast}:
\begin{definition}
For a matrix \(\mA\in\bbR^{n \times d}\), we define the \(\eps\)-ridge leverage score for row \(i\) as
\[
	\tau_{i,\eps}(\mA)
	\defeq
	\max_{\{\valpha\in\bbR^d:\normof{\valpha}_2 > 0\}}
	\frac{\abs{[\mA\valpha]_i}^2}{\normof{\mA\valpha}_2^2 + \eps\normof{\valpha}_2^2}
\]
\end{definition}

We first import a result from \cite{cohen2017input} that shows how ridge leverage score sampling spectrally embeds matrices:
\begin{importedtheorem}[Theorem 5 from \cite{cohen2017input}]
\label{impthm:ridge-spectral}
Let \(\mA\in\bbR^{n \times d}\) and \(\eps \geq 0\).
Let rows \(r_1,\ldots,r_m\) be sampled iid proportionally to \(\tilde\tau_\eps(i)\), where \(\tilde\tau_\eps(i) \geq \tau_{i,\eps}(\mA)\).
Define \(\tilde s \defeq \sum_{i=1}^m \tilde\tau_\eps(i)\).
Let \(\mS\in\bbR^{m \times n}\) be the sample and rescale matrix: \([\mS]_{i,j}=\sqrt{\frac{s}{m\tilde\tau_\eps(i)}} \cdot \mathbbm{1}_{[r_i = j]}\).
Then if \(m=O(\frac{s\log(s/\delta)}{\Delta^2})\), with probability \(1-\delta\) we have
\[
	(1-\Delta)(\mA^\intercal\mA - \eps\eye)
	\preceq
	\mA^\intercal\mS^\intercal\mS\mA + \eps\eye
	\preceq
	(1+\Delta)(\mA^\intercal\mA + \eps\eye)
\]
\end{importedtheorem}

Then we move onto the theorem we want to prove:

\begin{theorem}
\label{thm:multiple-designs}
Let \(\mA_1,\ldots,\mA_Q\in\bbR^{m \times d}\) and \(\vb\in\bbR^d\).
Fix ridge parameter \(\eps\geq0\).
Sample rows \(r_1,\ldots,r_n\propto\tilde\tau_\eps\) where \(\tilde\tau_\eps\) is an upper bound for the \(\eps\)-ridge leverage scores of all pairs of design matrices conjoined: \(\tilde\tau_\eps(i) \geq \tau_{i,\eps}([\mA_j, \mA_k])\) for all \(j,k\).
Let \(\tilde s_\eps = \sum_{i=1}^m \tilde\tau_\eps(i)\).
Build a sample-and-rescale matrix \(\mS\in\bbR^{n \times m}: \mS_{j,k} = \sqrt{\frac{s}{n \tilde\tau_\eps}} \mathbbm{1}_{[r_j=k]}\).
Then let \(\tilde k,\tilde\vx\) solve the subsampled regression problem:
\[
	\tilde k, \tilde\vx \defeq \argmin_{k\in[Q],\vx\in\bbR^d}
	\normof{\mS\mA_k\vx - \mS\vb}_2^2 + \eps\normof{\vx}_2^2
\]
If \(n = O(\tilde s_\eps \log(\frac{\tilde s_\eps}{\delta} \cdot Q))\), then with probability \(1-\delta\) we have
\[
	\normof{\mA_{\tilde k}\tilde\vx - \vb}_2^2
	\leq (9+\nicefrac8\delta)
	\min_{k\in[Q]}\min_{\vx\in\bbR^d}
	\normof{\mA_k\vx-\vb} + \eps\normof{\vx}_2^2
\]
\end{theorem}

The proof of \theoremref{multiple-priors} closely mirrors that of \theoremref{multiple-designs}, except that \theoremref{multiple-priors} additionally bounds the Fourier version of the pairwise leverage scores \(\tau_{i,\eps}([\mA_j, \mA_k])\) and proves a new operator spectral embedding guarantee to handle this case.

\begin{proof}
Let \(\hat k\) and \(\hat \vx\) be the true minimizers for the full optimization problem:
\[
	\hat k, \hat \vx \defeq \argmin_{k\in[Q], \vx\in\bbR^d} \normof{\mA_{k}\vx - \vb}_2 + \eps\normof{\vx}_2^2
\]
By the triangle inequality, and the inverse triangle inequality, we have for any \(\mA_k
\) and any \vx,
\[
	\normof{\mS(\mA_k\vx-\vb)}_2 \in \normof{\mS(\mA_{\hat k}\hat\vx - \mA_k\vx)}_2 \pm \normof{\mS(\mA_{\hat k}\hat\vx - \vb)}_2
\]
Letting \(\hat\vb_\perp\defeq \mA_{\hat k}\hat\vx-\vb\), we have
\[
	\E[\normof{\mS(\mA_{\hat k}\hat\vx - \vb)}_2^2]
	= \E[\normof{\mS\hat\vb_\perp}_2^2]
	= \E[\hat\vb_\perp^\intercal \mS^\intercal \mS \hat\vb_\perp]
	= \hat\vb_\perp^\intercal \E[\mS^\intercal\mS] \hat\vb_\perp
\]
And since \(\E[\mS^\intercal\mS]=\eye\), we find \(\E[\normof{\mS(\mA_{\hat k}\hat\vx - \vb)}_2^2] = \normof{\hat\vb_\perp}_2^2 = \normof{\mA_{\hat k}\hat\vx - \vb}_2^2\).
Hence, by Markov's inequality, we have
\[
	\normof{\mS(\mA_{\hat k}\hat\vk - \vb)}_2^2 + \eps\normof{\hat\vx}_2^2
	\leq
	\frac2\delta (\normof{\mA_{\hat k}\hat\vx - \vb}_2^2 + \eps\normof{\hat\vx}_2^2)
\]
with probability \(1-\frac\delta2\).

Next, note that for any \(\mA_j,\mA_k\) we have \(\normof{\mS[\mA_j ~ \mA_k]\vv}_2^2 +\eps\normof{\vv}_2^2 \in (1\pm\Delta)(\normof{[\mA_j ~ \mA_k]\vv}_2^2+\eps\normof{\vv}_2^2)\) for all \(\vv\in\bbR^{2d}\) with probability \(1-\frac\delta2\).
This follows directly from the fact that \(\mS\) is generated using upper bounds for leverage scores for \([\mA_j ~ \mA_k]\), following \importedtheoremref{ridge-spectral}.
Then, we find
\begin{align*}
	\normof{\mS(\mA_{\hat k}\hat\vx - \mA_k\vx)}_2^2 + \eps\normof{\hat\vx}_2^2 + \eps\normof{\vx}_2^2
	&= \normof{\mS \bmat{\mA_{\hat k} & \mA_k} \sbmat{\hat\vx \\ -\vx}}_2^2 + \eps\normof{\sbmat{\hat\vx \\ -\vx}}_2^2 \\
	&\in (1\pm\Delta) (\normof{\bmat{\mA_{\hat k} & \mA_k} \sbmat{\hat\vx \\ -\vx}}_2^2 + \eps\normof{\sbmat{\hat\vx \\ -\vx}}_2^2)\\
	&= (1\pm\Delta) (\normof{\mA_{\hat k}\hat\vx - \mA_k\vx}_2^2 + \eps\normof{\hat\vx}_2^2 + \eps\normof{\vx}_2^2)\\
\end{align*}
Further, by the triangle inequality, we have
\[
	\normof{\mA_{\hat k}\hat\vx - \mA_k\vx}_2^2
	\leq
	\normof{\mA_{\hat k}\hat\vx - \vb}_2^2 + \normof{\mA_{k}\vx - \vb}_2^2
\]
Putting these last two inequalities together, we find that all \(\mA_k\) and \vx have
\[
	\normof{\mS(\mA_{\hat k}\hat\vx - \mA_k\vx)}_2^2  + \eps\normof{\hat\vx}_2^2 + \eps\normof{\vx}_2^2
	\in
	(1\pm\Delta)\left(\normof{\mA_{\hat k}\hat\vx - \vb}_2^2 + \eps\normof{\hat\vx}_2^2 + \normof{\mA_{k}\vx - \vb}_2^2 + \eps\normof{\vx}_2^2\right)
\]
Then, using this bound, alongside the Markov bound and the original triangle inequality, we find
\begin{align*}
	\normof{\mS(\mA_k\vx-\vb)}_2^2 + \eps\normof{\vx}_2^2
	&\in (\normof{\mS(\mA_{\hat k}\hat\vx - \mA_k\vx)}_2^2 + \eps\normof{\hat\vx}_2^2 + \eps\normof{\vx}_2^2) \pm (\normof{\mS(\mA_{\hat k}\hat\vx - \vb)}_2^2 + \eps\normof{\vx}_2^2) \\
	&\in (\normof{\mS(\mA_{\hat k}\hat\vx - \mA_k\vx)}_2^2 + \eps\normof{\hat\vx}_2^2 + \eps\normof{\vx}_2^2) \pm \frac2\delta (\normof{\mA_{\hat k}\hat\vx - \vb}_2^2 + \eps\normof{\vx}_2^2) \\
	&\in (1\pm\Delta)\left(\normof{\mA_{\hat k}\hat\vx - \vb}_2^2 + \eps\normof{\hat\vx}_2^2 + \normof{\mA_{k}\vx - \vb}_2^2 + \eps\normof{\vx}_2^2\right) \pm \frac2\delta (\normof{\mA_{\hat k}\hat\vx - \vb}_2^2 + \eps\normof{\vx}_2^2) \\
	&\in (1\pm\Delta)(\normof{\mA_{k}\vx - \vb}_2^2 + \eps\normof{\vx}_2^2) \pm \left(1+\Delta+\frac2\delta\right) (\normof{\mA_{\hat k}\hat\vx - \vb}_2^2 + \eps\normof{\hat\vx}_2^2) \\
\end{align*}
Note that the above bound holds for any choice of \(\mA_k\) and any \vx.
To simplify the constants a bit, let \(c_0\defeq \left(1+\Delta+\frac2\delta\right)\), \(\cL(k,\vx) \defeq \normof{\mA_k\vx - \vb}_2^2 + \eps\normof{\vx}_2^2\), and \(L(k,\vx)\defeq \normof{\mS(\mA_k\vx-\vb)}_2^2 + \eps\normof{\vx}_2^2\).
Then, the previous bound state that
\[
	L(k,g) \in (1\pm\Delta)\cL(k,g) \pm c_0 \cL(\hat k, \hat \vx)
\]
If we take \(k = \tilde k\) and \(\vx=\tilde\vx\), and rearrange terms, we find
\begin{align*}
	\cL(\tilde k, \tilde\vx)
	&\leq \frac{1}{1-\Delta} L(\tilde k, \tilde\vx) + \frac{c_0}{1-\Delta} \cL(\hat k, \hat\vx) \\
	&\leq \frac{1}{1-\Delta} L(\hat k, \hat\vx) + \frac{c_0}{1-\Delta} \cL(\hat k, \hat\vx) \\
	&\leq \frac{1}{1-\Delta} \left((1+\Delta)\cL(\hat k, \hat\vx) + c_0\cL(\hat k, \hat\vx)\right) + \frac{c_0}{1-\Delta} \cL(\hat k, \hat\vx) \\
	&= \frac{1+\Delta+2c_0}{1-\Delta}\cL(\hat k, \hat\vx) \\
	&= \frac{1+\Delta+2(1+\Delta+\frac2\delta)}{1-\Delta}\cL(\hat k, \hat\vx) \\
	&= \frac{3+3\Delta+\frac4\delta}{1-\Delta}\cL(\hat k, \hat\vx)
\end{align*}
If we take \(\Delta=\frac12\), we complete the proof.
\end{proof}

\subsection{Proof for the Operator Case}
\label{app:multiple-priors}

We start with preliminary definitions for randomized operator analysis.

\subsubsection{Ridge Leverage Scores}
To achieve near optimal sample complexity for kernel interpolation (i.e within logarithmic factors of the statistical dimension), recent work shows that it suffices to select time samples independently at random, according to a carefully chosen non-uniform distribution \cite{ChenKanePrice:2016,avron2019universal}.
In particular, we use the well studied ridge leverage function \cite{alaoui2015fast,MuscoMusco:2017,PauwelsBachVert:2018}, which is defined as follows:

\begin{definition}[Ridge leverage function]
\label{def:ridgeScores}
	For any Hilbert space $\cH$, time length $T > 0$, $\eps \geq 0$, and bounded linear operator $\cA: \cH \rightarrow L_2(T)$ the $\eps$-ridge leverage function for $t \in [0,T]$ is:
	\begin{align}
	\label{eq:leverage_def}
	\tau_{\cA,\eps}(t) = \frac{1}{T} \cdot \max_{\{\alpha \in \cH:\, \normof{\alpha}_{\cH} > 0\}} \frac{\left|[\cA \alpha](t) \right|^2 }{\normof{\cA \alpha}_T^2 + \eps \normof{\alpha}_\cH^2}.
	\end{align}
\end{definition}
Note that when $\cA$ is an inverse Fourier transform operator, the integral of the ridge leverage function is equal to the statistical dimension of the corresponding kernel -- i.e. if $\cA = \cF_\mu^*$ then $s_{\mu,\eps} = \int_0^T \tau_{\cA,\eps}(t)dt$. This fact generalizes a well known claim for matrices and is proven in \cite{avron2019universal}.
The ridge leverage score captures how important a time point $t$ is for $\cA$: it is large if there are low energy functions (small $\|\alpha\|_\cH^2$) in the span of the operator that are highly concentrated at $t$ -- i.e. when the function $\cA\alpha$ has large magnitude at $t$ compared to its average magnitude over $[0,T]$. 

The Universal Sampling Distribution (\defref{univ_dist}) is called \emph{Universal} because when \(\cA = \cF_\mu^*\) is \emph{any} inverse Fourier transform operator, recent work \cite{ChenPrice:2019a,avron2019universal} shows that \(\tau_{\cA,\eps}\) is tightly upper bounded by \(\tilde\tau_\alpha\):
\begin{claim}[Theorem 17 of \cite{avron2019universal}]
\label{clm:lev_score_ub}
For any PDF \(\mu\) and corresponding inverse Fourier transform operator \(\cF_\mu^*\),
\[
	\tau_{\cF_\mu^*, \eps} (t) \leq \tilde\tau_\alpha(t)
\]
for all \(t\in[0,T]\), as long as \(\alpha \geq c s_{\mu,\eps}\) for some universal constant \(c > 0\).
\end{claim}

We then state a known operator subsampling result from \cite{avron2019universal} which is based on the ridge leverage scores of \defref{ridgeScores}.
The proof of this result adapts a bound on sums of random operators by \cite{Minsker:2017}, and uses the upper bound of \claimref{lev_score_ub}.
A similar result is proven in \cite{Bach:2017}.
\begin{lemma}[Lemma 43 in \cite{avron2019universal}]
\label{lem:op_concentration}
	Consider a bounded linear operator $\cA: \cH \rightarrow L_2(T)$.
	Let $\tilde{\tau}_{\cA,\eps}(t)$ be a function with $\tilde{\tau}_{\cA,\eps}(t)\geq {\tau}_{\cA,\eps}(t)$ for all $t\in [0,T]$ and let $\tilde{s}_{\cA,\eps} = \int_0^T \tilde{\tau}_{\cA,\eps}(t)dt$. Let $n = c \cdot \Delta^{-2}\tilde{s}_{\cA,\eps} \log(\tilde{s}_{\cA,\eps}/\delta)$ for sufficiently large fixed constant $c$ and select $t_1, \ldots, t_n$ by drawing each randomly from $[0,T]$ with probability proportional to $\tilde{\tau}_{\cA,\eps}(t)$. For $j\in 1, \ldots, s$, let $w_j = \sqrt{\frac{\tilde{s}_{\cA,\eps}}{nT\cdot\tilde{\tau}_{\cA,\eps}(t_j)}}$. 
	Let $\bA: \cH \rightarrow \bbC^n$ be the operator defined by $[\bA g]_j = [\cA g](t_j) \cdot w_j$. 
	With probability $(1-\delta)$,
	\begin{align*}
	(1-\Delta)(\cG + \eps \cI_\cH) \preceq \bA^*\bA + \eps \cI_\cH \preceq (1+\Delta)(\cG + \eps \cI_\cH).
	\end{align*}
\end{lemma}

\subsubsection{Concentration of Concatenated Fourier Operators}

With \lemmaref{op_concentration} in place, our goal in this section is prove a specific approximation result for randomly subsampling rows from the  the \emph{concatenation} of two inverse Fourier transform operators, $\cF_{\mu_1}^*$ and $\cF_{\mu_2}^*$. Specifically, let $\oplus$ denote the standard direct sum operation between Hilbert spaces. I.e. $[\alpha,\beta]\in \cH_1 \oplus \cH_2$ if $\alpha \in \cH_1$ and $\beta \in \cH_1$. For finitely bounded PDFs $\mu_1$ and $\mu_2$ the concatenated operator $\cF_{\mu_1,\mu_2}: L_2(\mu_1)\oplus L_2(\mu_2) \rightarrow L_2(T)$ is defined as:
\begin{align*}
\cF_{\mu_1,\mu_2}^*[\alpha,\beta]= \cF_{\mu_1}^*\alpha + \cF_{\mu_2}^*\beta.
\end{align*}
Note that the adjoint of $\cF_{\mu_1,\mu_2}^*$ is $\cF_{\mu_1,\mu_2} f = (\cF_{\mu_1} f, \cF_{\mu_2} f)$. 

Our goal is to approximate $\cF_{\mu_1,\mu_2}^*$ by an operator with a finite number of rows. Such an approximation could be obtained directly from \lemmaref{op_concentration}. However, applying that result requires \emph{an upper bound on the ridge leverage scores (\defref{ridgeScores}) of $\cF_{\mu_1,\mu_2}^*$}. Our first technical result of this section is to show that such an upper bound can be obtained using the universal sampling distribution of \defref{univ_dist}. We prove:

\begin{lemma}
	\label{lem:pairwise_lev_score_bound}
	For any bounded PDFs $\mu_1,\mu_2$ on $\bbR$ let, $\cA = \cF_{\mu_1,\mu_2}^*$ where $\cF_{\mu_1,\mu_2}^*$ is a concatenated inverse Fourier transform operator as defined above for any $\eps > 0$,
	\begin{align*}
		\tau_{\cA,\eps}(t)  \leq \tilde\tau_\alpha(t)
	\end{align*}
	as long as $\alpha \geq c \cdot \max\left[s_{\mu_1,\eps},s_{\mu_2,\eps}\right]$ for some fixed constant $c$.
\end{lemma}
\begin{proof}
	Let $\bar{\mu} = \frac{\mu_1+\mu_2}{2}$ and let $\bar{\cA} = 2\cF_{\bar{\mu}}^*$. We establish the lemma by proving
	\begin{align}
	\label{eq:bound_by_sum}
		\tau_{\cA,\eps}(t) \leq \tau_{\bar{\cA},\eps}(t) 
	\end{align}
	Once we have this bound, we can apply \claimref{lev_score_ub} to observe that $\tau_{\bar{\cA},\eps}(t) \leq \tilde\tau_\alpha(t)$ as long as long as $\alpha \geq c \cdot s_{\bar{\mu},\eps}$. Finally, from Lemma 51 in \cite{avron2019universal}, we have that $s_{\bar{\mu},\eps} \leq 2\max\left[s_{\mu_1,\eps},s_{\mu_2,\eps}\right]$, which gives the lemma because $\tilde\tau_\alpha(t)$ is strictly increasing with $\alpha$.
	
	So, we are left to prove \eqref{eq:bound_by_sum}. Referring to  \defref{ridgeScores} and noting that $\normof{[\alpha,\beta]}_{\cH_1 \oplus \cH_2}^2 = \normof{\alpha}_{\cH_1}^2 + \normof{\beta}_{\cH_2}^2$ , we can do so by upper bounding for all $t \in [0,T]$:
	\begin{align}
	\label{eq:leverage_def_ds}
	\frac{1}{T} \cdot \max_{\{[\alpha,\beta]\in L_2(\mu_1)\oplus L_2(\mu_2):\, \normof{\alpha}_{\mu_1}^2 + \normof{\beta}_{\mu_2}^2 > 0\}} \frac{\left|[\cF_{\mu_1,\mu_2}^* [\alpha,\beta]](t) \right|^2 }{\normof{\cF_{\mu_1,\mu_2}[\alpha,\beta]}_T^2 + \eps \normof{\alpha}_{\mu_1}^2 + \eps \normof{\beta}_{\mu_2}^2}.
	\end{align}
	
	For any particular $t\in [0,T]$, let $\alpha^* \in L_2(\mu_1)$ and $\beta^* \in L_2(\mu_1)$ be the maximizers of \eqref{eq:leverage_def_ds}.
	 We are going to define a function $w$ to satisfy $\bar{\cA} w = \cF_{\mu_1,\mu_2}^*[\alpha^*, \beta^*]$.  
	In particular, we can set 
	\begin{align*}
	w(\xi) = (\mu_1(\xi) + \mu_2(\xi) )^+\left(\mu_1(\xi)\alpha^*(\xi) + \mu_2(\xi)\beta^*(\xi)\right)
	\end{align*}
	where for a $s\in \bbR$, $s^+$ evaluates to $0$ when $s = 0$ and $1/s$ otherwise.
	We have that \eqref{eq:leverage_def_ds} is equal to:
	\begin{align}
		\label{eq:leverage_def_ds_2}
	\frac{1}{T} \cdot  \frac{\left|[\bar{\cA} w](t) \right|^2 }{\normof{\bar{\cA} w}_T^2 + \eps \normof{\alpha^*}_{\mu_1}^2 + \eps \normof{\beta^*}_{\mu_2}^2}.
	\end{align}
	
	Next we bound $\normof{w}_{\bar{\mu}}^2 = \int_{\xi \in \bbR} w(\xi)^2 \bar{\mu}(\xi) d\xi$. We have that for all $\xi$,
	\begin{align*}
	w(\xi)^2 \bar{\mu} &=  \frac{1}{2}(\mu_1(\xi) + \mu_2(\xi) )^+ \cdot \left(\mu_1(\xi)\alpha^*(\xi) + \mu_2(\xi)\beta^*(\xi)\right)^2\\
	&\leq (\mu_1(\xi) + \mu_2(\xi) )^+ \cdot \left(\mu_1(\xi)^2\alpha^*(\xi)^2 + \mu_2(\xi)^2\beta^*(\xi)^2\right)\\
	&\leq \mu_1(\xi)\alpha^*(\xi)^2 + \mu_2(\xi)\beta^*(\xi)^2.
	\end{align*}
	It follows that $\int_{\xi \in \bbR} w(\xi)^2 \bar{\mu}(\xi) d\xi \leq \int_{\xi \in \bbR} \alpha^*(\xi)^2 \mu_1(\xi) d\xi + \int_{\xi \in \bbR} \beta^*(\xi)^2 \mu_2(\xi) d\xi = \normof{\alpha^*}_{\mu_1}^2 + \normof{\beta^*}_{\mu_2}^2$.
	Substituting into \eqref{eq:leverage_def_ds_2}, we actually have that \eqref{eq:leverage_def_ds} can be upper bounded by
	\begin{align*}
	\frac{1}{T} \cdot  \frac{\left|[\bar{\cA} w](t) \right|^2 }{\normof{\bar{\cA} w}_T^2 + \eps \normof{w}_{\bar{\mu}}^2}.
	\end{align*}
	This quantity is of course only small than $\tau_{\bar{\cA},\eps}(t)$, which completes the proof of \eqref{eq:bound_by_sum}.
\end{proof}

The following theorem is a direct corollary of \lemmaref{op_concentration} and  \lemmaref{pairwise_lev_score_bound}.

\begin{theorem}
	\label{thm:gram-operator-spectral}
	Fix \(\Delta > 0\) and \(\delta > 0\).
	Let \(\mu_1, \mu_2\) be bounded PDFs. Let $s_{max} = \max\left[s_{\mu_1,\eps}, s_{\mu_2,\eps}\right]$. Let $\alpha = c_0 s_{max}$ and $n = c_1\Delta^{-2} s_{max}\log (s_{max})\log(s_{max}/\delta)$ for fixed universal constants $c_0,c_1$. Suppose $n$ time samples $t_1,\ldots, t_n \in [0,T]$ are sampled with probability proportional to $\tilde{\tau}_{\alpha}(t)$ and \(\mF_{\mu_1}^*\) and \(\mF_{\mu_2}^*\) be the sampled versions of \(\cF_{\mu_1}^*\) and \(\cF_{\mu_2}^*\) satisfying for $j = 1,\ldots, n$:
	\begin{align*}
	[\mF_{\mu_p}^*g]_j = w_j\cdot\int_\bbR g(\xi) e^{2 \pi i \xi t_j} \mu_p(\xi) d\xi,
	\end{align*}
	where  $w_j = \sqrt{\frac{\int_0^T\tilde{\tau}_{\alpha}(t)dt}{sT\cdot\tilde{\tau}_{\alpha}(t_j)}}$.
	Then with probability $(1-\delta)$,
		\[
	(1-\Delta) (\cG + \eps\cI) \preceq \tilde \cG + \eps\cI \preceq (1+\Delta) (\cG + \eps\cI)
	\]
	where $\cG = \cF_{\mu_1,\mu_2}\cF_{\mu_1,\mu_2}^*$ and $\bar{\cG} = [\mF_{\mu_2};\mF_{\mu_1}] [\mF_{\mu_2}^*,\mF_{\mu_1}^*]$. 
	Here $[\mF_{\mu_2};\mF_{\mu_1}]: \bbC^s \rightarrow L_2(\mu_1) \oplus L_2(\mu_2)$ 
	is the natural concatenation of $\mF_{\mu_2}$ and $\mF_{\mu_1}$, and $[\mF_{\mu_2}^*,\mF_{\mu_1}^*]$ is the concatenation of $\mF_{\mu_2}^*$ and $\mF_{\mu_1}^*$.
	\begin{proof}
		By \lemmaref{pairwise_lev_score_bound} $\tilde{\tau}_{\alpha}(t)$ strictly upper bounds the $\eps$-ridge leverage scores of $\cF_{\mu_1,\mu_2}^*$ as long as $\alpha$ is set as in the theorem statement. Moreover, referring to \defref{univ_dist}, $\int_0^T \tilde{\tau}_{\alpha}(t) dt \leq O(\alpha \log \alpha)$, so the number of samples $n$ in the theorem is sufficiently large to directly apply \lemmaref{op_concentration} to the bounded linear operator $\cF_{\mu_1,\mu_2}^*$.
	\end{proof}
	
\end{theorem}

\subsubsection{Final Result for Linear Operators}

\begin{reptheorem}{multiple-priors}
Let \(\tilde\cU = \{\mu_1,\ldots,\mu_Q\}\) be a finite set of scaled PDFs.
Let \(s_{max,\eps}\) be the maximum of the PDFs' statistical dimensions: \(s_\eps = \max_j s_{\mu_j, \eps}\).
Let \(t_1,\ldots,t_n\) be iid samples from the universal sampling distribution, and define \(\mF^*\) accordingly.
Let \(\tilde\mu,\tilde g\) optimally solve the time-discretized problem:
\[
	\tilde\mu, \tilde g
	\defeq
	\argmin_{\mu\in\tilde\cU, g\in L_2(\mu)}
	\normof{\mF_{\tilde\mu}^* g - \bar\vy}_2^2 + \eps\normof{g}_{\mu}^2
\]
If \(n = O(s_{\eps} \log(\frac{s_\eps ~ Q}{\delta}))\), then with probability \(1-\delta\), we have
\[
	\normof{\cF_{\tilde\mu}^* \tilde g - \bar y}_T^2 + \eps\normof{\tilde g}_{\tilde\mu}^2
	\leq
	(9+\nicefrac{8}{\delta}) \cdot
	\argmin_{\mu\in\tilde\cU, g\in L_2(\mu)}
	\normof{\cF_{\mu}^* g - \bar y}_T^2 + \eps\normof{g}_{\tilde\mu}^2
\]
\end{reptheorem}

\begin{proof}
Let \(\hat\mu\) and \(\hat \vx\) be the true minimizers for the full optimization problem:
\[
	\hat\mu, \hat \vx \defeq \argmin_{\mu\in\tilde\cU, g\in L_2(\mu)} \normof{\cF_\mu^* g - \bar y}_T^2 + \eps\normof{g}_\mu^2
\]
By the triangle inequality, and the inverse triangle inequality, we have for any \(\mu\) and any \(g\in L_2(\mu)\),
\[
	\normof{\mF_\mu^* g - \bar\vy}_2^2 \in \normof{\mF_{\hat\mu}^* \hat g - \mF_\mu^* g}_2^2 \pm \normof{\mF_{\hat\mu}^* \hat g - \bar\vy}_2^2
\]
Note from \cite{avron2019universal} that \(\E[\normof{\mF_{\hat\mu}^* \hat g - \bar\vy}_2^2] = \normof{\cF_{\hat\mu}^* \hat g - \bar y}_T^2\).
Hence, by Markov's inequality, we have
\[
	\normof{\mF_{\hat\mu}^* \hat g - \bar\vy}_2^2 + \eps\normof{\hat g}_{\hat\mu}^2 \leq \frac2\delta \left(\normof{\cF_{\hat\mu}^* \hat g - \bar y}_T^2 + \eps\normof{\hat g}_{\hat\mu}^2\right)
\]
with probability \(1-\frac\delta2\).

Next, note that for any \(\mu_j,\mu_k\in\tilde\cU\) we have
\[
	\normof{\mF_{\mu_j}^* g + \mF_{\mu_k}^* h}_2^2 + \eps\normof{g}_{\mu_j}^2 + \eps\normof{h}_{\mu_k}^2
	\in (1\pm\Delta)
	\left(
		\normof{\cF_{\mu_j}^* g + \cF_{\mu_k}^* h}_T^2 + \eps\normof{g}_{\mu_j}^2 + \eps\normof{h}_{\mu_k}^2
	\right)
\] for all \(g \in L_2(\mu_j), h\in L_2(\mu_k)\) with probability \(1-\frac\delta2\).
This follows directly from \theoremref{gram-operator-spectral}.

\noindent Further, by the triangle inequality, we have for any \(\mu\in\tilde\cU,g\in L_2(\mu)\)
\[
	\normof{\cF_{\hat\mu}^*\hat g - \cF_\mu^* g}_T^2
	\leq
	\normof{\cF_{\hat\mu}^*\hat g - \bar y}_T^2 + \normof{\cF_\mu\vx^* - \bar y}_T^2
\]
Putting these last two inequalities together, we find
\[
	\normof{\mF_{\hat\mu}^* \hat g + \mF_{\mu}^* g}_2^2 + \eps\normof{\hat g}_{\hat\mu}^2 + \eps\normof{g}_{\mu}^2
	\in	(1\pm\Delta)
	\left(
		\normof{\cF_{\hat\mu}^*\hat g - \bar y}_T^2 + \eps\normof{\hat g}_{\hat\mu}^2 + \normof{\cF_{\mu}^* g - \bar y}_T^2 + \eps\normof{g}_{\mu}^2
	\right)
\]
Then, using this bound, alongside the Markov bound and the original triangle inequality, we find
\begin{align*}
	\normof{\mF_{\mu}^* g - \bar\vy}_2^2 + \eps\normof{g}_{\mu}^2
	&\in \left(\normof{\mF_{\hat\mu}^* \hat g - \mF_\mu^* g}_2^2 + \eps\normof{\hat g}_{\hat\mu}^2 + \eps\normof{g}_\mu^2\right) \pm \left(\normof{\mF_{\hat\mu}^* \hat g - \bar\vy}_2^2 + \eps\normof{\hat g}_{\hat\mu}^2\right) \\
	&\in \left(\normof{\mF_{\hat\mu}^* \hat g - \mF_\mu^* g}_2^2 + \eps\normof{\hat g}_{\hat\mu}^2 + \eps\normof{g}_\mu^2\right) \pm \frac2\delta \left(\normof{\cF_{\hat\mu}^* \hat g - \bar y}_T^2 + \eps\normof{\hat g}_{\hat\mu}^2\right) \\
	&\in (1\pm\Delta)\left(\normof{\cF_{\hat\mu}^*\hat g - \bar y}_T^2 + \eps\normof{\hat g}_{\hat\mu}^2 + \normof{\cF_{\mu}^* g - \bar y}_T^2 + \eps\normof{g}_{\mu}^2\right) \pm \frac2\delta \left(\normof{\cF_{\hat\mu}^* \hat g - \bar y}_T^2 + \eps\normof{\hat g}_{\hat\mu}^2\right) \\
	&\in (1\pm\Delta)\left(\normof{\cF_{\mu}^* g - \bar y}_T^2 + \eps\normof{g}_{\mu}^2\right) \pm \left(1+\Delta+\frac2\delta\right) \left(\normof{\cF_{\hat\mu}^* \hat g - \bar y}_T^2 + \eps\normof{\hat g}_{\hat\mu}^2\right)
\end{align*}
To simplify the notation a bit, let \(c_0\defeq \left(1+\Delta+\frac2\delta\right)\), \(\cL(\mu,g) \defeq \normof{\cF_{\mu}^* g - \bar y}_T^2 + \eps\normof{g}_{\mu}^2\), and \(L(\mu,g) \defeq \normof{\mF_{\mu}^* g - \bar\vy}_2^2 + \eps\normof{g}_{\mu}^2\).
Then the previous bound says
\[
	L(\mu,g) \in (1\pm\Delta)\cL(\mu,g) \pm c_0 \cL(\hat \mu, \hat g)
\]
Recall that this bound holds for any choice of \(\mu\in\tilde\cU\) and any \(g\in L_2(\mu)\).
If we take \(\mu = \tilde\mu\) and \(g = \tilde g\), and rearrange terms, we find
\begin{align*}
	\cL(\tilde \mu, \tilde g)
	&\leq \frac{1}{1-\Delta} L(\tilde \mu, \tilde g) + \frac{c_0}{1-\Delta} \cL(\hat \mu, \hat g) \\
	&\leq \frac{1}{1-\Delta} L(\hat \mu, \hat g) + \frac{c_0}{1-\Delta} \cL(\hat \mu, \hat g) \\
	&\leq \frac{1}{1-\Delta} \left((1+\Delta)\cL(\hat \mu, \hat g) + c_0\cL(\hat\mu, \hat g)\right) + \frac{c_0}{1-\Delta} \cL(\hat \mu, \hat g) \\
	&= \frac{1+\Delta+2c_0}{1-\Delta}\cL(\hat \mu, \hat g) \\
	&= \frac{1+\Delta+2(1+\Delta+\frac2\delta)}{1-\Delta}\cL(\hat \mu, \hat g) \\
	&= \frac{3+3\Delta+\frac4\delta}{1-\Delta}\cL(\hat \mu, \hat g)
\end{align*}
If we take \(\Delta=\frac12\), we complete the proof.
\end{proof}

\end{document}